%% file: iclr2026_conference.tex
\documentclass{article} 
\usepackage{iclr2026_conference,times}

\usepackage{etoolbox}
\newtoggle{iclr}
\toggletrue{iclr}

\newcommand{\arxiv}[1]{\iftoggle{iclr}{}{#1}}

\input{header.tex}

\input{characters.tex}
\input{macros.tex}

\title{From Curiosity to Caution: Mitigating \\ Reward Hacking for Best-of-$N$ with Pessimism}

\author{
    \centerline{\textbf{
        Zhuohao Yu\textsuperscript{{1}}, \quad
        Zhiwei Steven Wu\textsuperscript{{1}},  \quad
        Adam Block\textsuperscript{{2}}
    }} \\
    \centerline{\textsuperscript{1}Carnegie Mellon University  \quad \quad
    \textsuperscript{2}Columbia University} \\
    \centerline{\texttt{zhuohaoy@andrew.cmu.edu, zstevenwu@cmu.edu, adam.block@columbia.edu}}
    \vspace{-.1cm}
}

\iclrfinalcopy
\begin{document}

\maketitle

\begin{abstract}
Inference-time compute scaling has emerged as a powerful paradigm for improving language model performance on a wide range of tasks, but the question of how best to use the additional compute remains open.  A popular approach is \emph{Best-of-$N$} (BoN) sampling, where $N$ candidate responses are generated, scored according to a reward model, and the highest-scoring response is selected.  While this approach can improve performance, it is vulnerable to \emph{reward hacking}, where performance degrades as $N$ increases due to the selection of responses that exploit imperfections in the reward model instead of genuinely improving generation quality. Prior attempts to mitigate reward hacking, via stronger reward models or heavy-handed distributional regularization, either fail to fully address over-optimization or are too conservative to exploit additional compute.  In this work, we explore the principle of \emph{pessimism} in RL, which uses lower confidence bounds on value estimates to avoid OOD actions with uncertain reward estimates. Our approach, termed as \emph{caution}, can be seen as the {reverse} of \emph{curiosity}: where curiosity (e.g., via Random Network Distillation, RND) rewards prediction error as a signal of novelty, caution penalizes prediction error as a signal of distributional uncertainty. Practically, caution trains an error model on typical responses and uses its prediction error to lower reward estimates for atypical ones. 
Our extensive empirical evaluation demonstrates that caution is a simple, computationally efficient approach that substantially mitigates reward hacking in BoN sampling.  We also provide a theoretical analysis in a simplified linear setting, which shows that caution provably improves over the standard BoN approach.  Together, our results not only establish caution as a practical solution to reward hacking, but also provide evidence that curiosity-based approaches can be a general OOD detection technique in LLM settings.
\end{abstract}

\input{sec1-intro}

\input{sec3-methodology}

\input{sec4-experiments}

\input{sec5-discussion}

\nocite{*}


\section*{Ethics Statement}\label{sec:ethics}
We affirm adherence to the ICLR Code of Ethics. This work only involve publicly available benchmark datasets under their respective licenses; we do not collect new human-subject data and process no personally identifiable information.

\section*{Reproducibility Statement}\label{sec:reproducibility}
To ensure reproducibility and transparency of the results within this paper, we document relevant hyperparameters and implementation details in the appendix. We will open-source our codebase along with detailed instructions and scripts.

\bibliography{iclr2026_conference}
\bibliographystyle{iclr2026_conference}
\appendix

\section{The Use of Large Language Models (LLMs)}\label{app:llm_usage}
LLMs were used assist refining the writing of this paper, including grammar correction, wording refinement, and formatting adjustments. We also use LLM agents to help with finding relevant work and implementing parts of our code. The use of AI tools does not affect the originality of the work or the authors' responsibility for the content.

\input{sec2-related}
\input{app_proofs}

\input{app_experiments}

\end{document}

%% file: header.tex
\arxiv{
\usepackage[numbers,sort&compress]{natbib}
}

\usepackage{boxedminipage}
\usepackage{multirow,nicefrac}
\usepackage{makecell,upgreek}
\usepackage{footnote}
\usepackage{longtable}
\usepackage{tablefootnote}
\usepackage[T1]{fontenc}
\usepackage{verbatim} 
\usepackage[utf8]{inputenc}
\usepackage{tikz}
\usepackage{booktabs}   
\usepackage{adjustbox}

\arxiv{
\usepackage[margin=0.75in]{geometry}
}
\usepackage[ruled,vlined]{algorithm2e}

\usepackage[dvipsnames]{xcolor}

\usepackage{subfigure}
\usepackage{wrapfig}

\usepackage{colortbl}
\definecolor{lightgray}{gray}{0.9}  

\usepackage{amsmath}
\usepackage{amsthm}

\usepackage[breaklinks=true]{hyperref}
\hypersetup{
	colorlinks=true,
	linkcolor=blue,
	citecolor=blue,
	urlcolor=blue
}

\usepackage[capitalise,nameinlink]{cleveref}
\usepackage{amsfonts,amssymb,mathrsfs}
\usepackage{mathtools}

\usepackage{appendix}

\theoremstyle{plain}
	\newtheorem{theorem}{Theorem}

	\newtheorem{lemma}{Lemma}

	\newtheorem{proposition}{Proposition}

	\theoremstyle{definition}

	\newtheorem{definition}{Definition}

\Crefname{appendix}{Appendix}{Appendices}

\usepackage{autonum}

\usepackage[T1]{fontenc}
\usepackage[scaled]{beramono}
\usepackage{listings}

\definecolor{darkred}{RGB}{139,0.0,0.0}
\definecolor{varorange}{RGB}{230,126,34}     
\definecolor{opblue}{RGB}{52,152,219}        
\definecolor{ifred}{RGB}{192,57,43}          
\definecolor{commentgray}{RGB}{150,150,150}  
\definecolor{codebg}{rgb}{0.98,0.98,0.98}    

\newsavebox\codebox

\arxiv{
\usepackage{authblk}
}

%% file: characters.tex
\renewcommand{\epsilon}{\varepsilon}

\def\ddefloop#1{\ifx\ddefloop#1\else\ddef{#1}\expandafter\ddefloop\fi}
\def\ddef#1{\expandafter\def\csname bb#1\endcsname{\ensuremath{\mathbb{#1}}}}
\ddefloop ABCDEFGHIJKLMNOPQRSTUVWXYZ\ddefloop

\def\ddefloop#1{\ifx\ddefloop#1\else\ddef{#1}\expandafter\ddefloop\fi}
\def\ddef#1{\expandafter\def\csname frak#1\endcsname{\ensuremath{\mathfrak{#1}}}}
\ddefloop ABCDEFGHIJKLMNOPQRSTUVWXYZ\ddefloop

\def\ddefloop#1{\ifx\ddefloop#1\else\ddef{#1}\expandafter\ddefloop\fi}
\def\ddef#1{\expandafter\def\csname fr#1\endcsname{\ensuremath{\mathfrak{#1}}}}
\ddefloop ABCDEFGHIJKLMNOPQRSTUVWXYZ\ddefloop

\def\ddefloop#1{\ifx\ddefloop#1\else\ddef{#1}\expandafter\ddefloop\fi}
\def\ddef#1{\expandafter\def\csname eul#1\endcsname{\ensuremath{\EuScript{#1}}}}
\ddefloop ABCDEFGHIJKLMNOPQRSTUVWXYZ\ddefloop

\def\ddefloop#1{\ifx\ddefloop#1\else\ddef{#1}\expandafter\ddefloop\fi}
\def\ddef#1{\expandafter\def\csname scr#1\endcsname{\ensuremath{\mathscr{#1}}}}
\ddefloop ABCDEFGHIJKLMNOPQRSTUVWXYZ\ddefloop

\def\ddefloop#1{\ifx\ddefloop#1\else\ddef{#1}\expandafter\ddefloop\fi}
\def\ddef#1{\expandafter\def\csname b#1\endcsname{\ensuremath{\mathbf{#1}}}}
\ddefloop ABCDEGHIJKLMNOPQRSTUVWXYZ\ddefloop

\def\ddefloop#1{\ifx\ddefloop#1\else\ddef{#1}\expandafter\ddefloop\fi}
\def\ddef#1{\expandafter\def\csname bhat#1\endcsname{\ensuremath{\hat{\mathbf{#1}}}}}
\ddefloop ABCDEFGHIJKLMNOPQRSTUVWXYZ\ddefloop

\def\ddefloop#1{\ifx\ddefloop#1\else\ddef{#1}\expandafter\ddefloop\fi}
\def\ddef#1{\expandafter\def\csname btil#1\endcsname{\ensuremath{\tilde{\mathbf{#1}}}}}
\ddefloop ABCDEFGHIJKLMNOPQRSTUVWXYZ\ddefloop

\def\ddefloop#1{\ifx\ddefloop#1\else\ddef{#1}\expandafter\ddefloop\fi}
\def\ddef#1{\expandafter\def\csname bst#1\endcsname{\ensuremath{\mathbf{#1}^\star}}}
\ddefloop ABCDEFGHIJKLMNOPQRSTUVWXYZ\ddefloop

\def\ddefloop#1{\ifx\ddefloop#1\else\ddef{#1}\expandafter\ddefloop\fi}
\def\ddef#1{\expandafter\def\csname bst#1\endcsname{\ensuremath{\mathbf{#1}^\star}}}
\ddefloop abcdeghijklmnopqrstuvwxyz\ddefloop

\def\ddefloop#1{\ifx\ddefloop#1\else\ddef{#1}\expandafter\ddefloop\fi}
\def\ddef#1{\expandafter\def\csname bhat#1\endcsname{\ensuremath{\hat{\mathbf{#1}}}}}
\ddefloop abcdefghijklmnopqrstuvwxyz\ddefloop

\def\ddefloop#1{\ifx\ddefloop#1\else\ddef{#1}\expandafter\ddefloop\fi}
\def\ddef#1{\expandafter\def\csname b#1\endcsname{\ensuremath{\mathbf{#1}}}}
\ddefloop abcdeghijklnopqrstuvwxyz\ddefloop

\def\ddefloop#1{\ifx\ddefloop#1\else\ddef{#1}\expandafter\ddefloop\fi}
\def\ddef#1{\expandafter\def\csname barb#1\endcsname{\ensuremath{\bar{\mathbf{#1}}}}}
\ddefloop abcdefghijklmnopqrstuvwxyz\ddefloop

\def\ddef#1{\expandafter\def\csname c#1\endcsname{\ensuremath{\mathcal{#1}}}}
\ddefloop ABCDEFGHIJKLMNOPQRSTUVWXYZ\ddefloop
\def\ddef#1{\expandafter\def\csname h#1\endcsname{\ensuremath{\widehat{#1}}}}
\ddefloop ABCDEFGHIJKLMNOPQRSTUVWXYZ\ddefloop
\def\ddef#1{\expandafter\def\csname hc#1\endcsname{\ensuremath{\widehat{\mathcal{#1}}}}}
\ddefloop ABCDEFGHIJKLMNOPQRSTUVWXYZ\ddefloop
\def\ddef#1{\expandafter\def\csname t#1\endcsname{\ensuremath{\widetilde{#1}}}}
\ddefloop ABCDEFGHIJKLMNOPQRSTUVWXYZ\ddefloop
\def\ddef#1{\expandafter\def\csname tc#1\endcsname{\ensuremath{\widetilde{\mathcal{#1}}}}}
\ddefloop ABCDEFGHIJKLMNOPQRSTUVWXYZ\ddefloop

%% file: macros.tex
\newcommand{\norm}[1]{\left\|#1\right\|}

\newcommand{\normf}[1]{\left\|#1\right\|_{\mathrm{F}}}
\newcommand{\abs}[1]{\left|#1\right|}
\newcommand{\inprod}[2]{\left\langle #1, #2 \right\rangle}

\newcommand{\rr}{\mathbb{R}}
\newcommand{\ee}{\mathbb{E}}

\DeclareMathOperator*{\argmax}{argmax}

\DeclareMathOperator{\Var}{Var}
\DeclareMathOperator{\Cov}{Cov}
\DeclareMathOperator{\trace}{Tr}

\newcommand{\thetastar}{\theta^{\star}}
\newcommand{\thetahat}{\hat{\theta}}
\newcommand{\piref}{\pi}
\newcommand{\rhat}{\hat{r}}
\newcommand{\rstar}{r^{\star}}
\newcommand{\wstar}{w^{\star}}
\newcommand{\what}{\hat{w}}

\newcommand{\projv}{\mathrm{proj}_V}
\newcommand{\projperp}{\mathrm{proj}_{V^{\perp}}}
\newcommand{\ihat}{\hat{i}}

\newcommand{\ipess}{i_{\mathrm{pess}}}
\newcommand{\istar}{i^{\star}}
\newcommand{\vperp}{V^{\perp}}
\newcommand{\relu}{\mathrm{ReLU}}
\newcommand{\Wstar}{W^{\star}}
\newcommand{\What}{\widehat{W}}
\newcommand{\lambdamin}{\lambda_{\mathrm{min}}}
\newcommand{\lambdamax}{\lambda_{\mathrm{max}}}
\newcommand{\wnought}{w^0}
\newcommand{\Wnought}{W^0}
\newcommand{\ustar}{u^{\star}}
\newcommand{\unought}{u^0}

\newcommand{\pihat}{\widehat{\pi}}
\newcommand{\rlcb}{\rhat_{\mathrm{LCB}}}

%% file: sec1-intro.tex
\section{Introduction}
Inference-time scaling has emerged as a transformative paradigm for enhancing language model performance, enabling significant improvements across a wide range of reasoning tasks without increasing model size \citep{brown2024large,guo2025deepseek,jaech2024openai}.  This success motivates the question of how best to leverage additional inference-time compute to maximize performance.  A particularly popular and effective approach is \emph{Best-of-$N$} (BoN) sampling \citep{stiennon2020learning, nakano2021webgpt, wang2022self, li2022competition,huang2025self},  where multiple candidate responses are generated for a given prompt, scored according to a reward model $\hat{r}$, and the highest-scoring response is selected.  This approach capitalizes on the intuition  that generating more candidates should increase the probability of finding higher-quality solutions, allowing the model to effectively `explore' a larger portion of the response space than it could with only a single response.

While BoN is a simple and competitive baseline that is capable of astonishing gains in many settings \citep{brown2024large}, its success is fundamentally limited by the quality of the reward model $\hat{r}$ used to score and select responses. Indeed, a common phenomenon often occurs with BoN, where performance initially improves as $N$ increases, but then hits an inflection point after which larger $N$ lead to increasingly worse outcomes \citep{gao2023scaling, huang2025best, khalaf2025inference}; an example of this phenomenon can be seen in \Cref{fig:Fig-1}, where we plot the performance of BoN on the mathematical reasoning task GSM8K for different $N$ scored by several different reward models.  This counterintuitive phenomenon, whereby increasing $N$ leads to worse performance, occurs due to \emph{reward hacking}, the process by which BoN selects responses that exploit $\rhat$ as opposed to the ground truth reward $\rstar$ we actually care about; while $\rhat$ may be a reasonable approximation to $\rstar$ on `typical' generations, as $N$ increases, BoN selects more atypical responses lying outside the training distribution of $\rhat$ that achieve high $\rhat$ scores through spurious correlations rather than genuine quality improvements.  In other words, reward hacking is a manifestation of Goodhart's Law: whenever a metric (in this case $\rhat$) becomes an optimization target, it ceases to be a reliable measure of quality \citep{goodhart1984monetary, weng2024rewardhacking}; for example, reward models are known to favor certain formatting preferences and surface-level patterns learned during training that do not necessarily correspond to improved reasoning or correctness \citep{liu2024rm,yu2025rewardanything,bukharin2025adversarial}.

\begin{figure}[t]
    \centering
    \includegraphics[width=0.85\columnwidth]{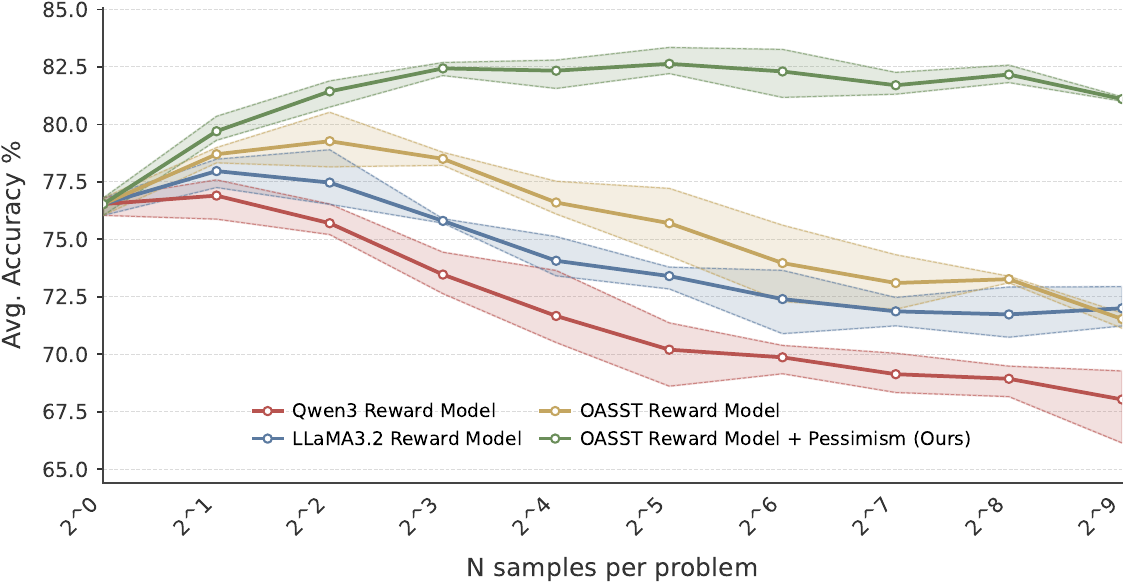}
    \caption{\textbf{Average Accuracy with different sampling budgets for Best-of-$N$} on the GSM8k dataset.  We see that standard Best-of-$N$ sampling (blue, red, and gold) suffers from reward hacking, exhibiting the characteristic rise-and-fall pattern as $N$ increases.  In contrast, caution (our approach, green) consistently improves with larger $N$, effectively mitigating reward hacking.}
    \label{fig:Fig-1}
\end{figure}

With the intuition that reward hacking is caused by out-of-distribution (OOD) generations from the perspective of the reward model, several works have proposed mitigation attempts by either improving the reward model \citep{liu2025skywork, yu2025rewardanything} or by regularizing the selection process to favor in-distribution responses \citep{huang2025best}.  Unfortunately, the former approach is fundamentally limited by the asymmetric difficulty of the problem: while it is relatively straightforward to obtain representative examples of high-quality responses through careful curation and human annotation, exhaustively characterizing all possible reward hacking strategies is intractable.  Representative of the latter work is \citet{huang2025best}, who propose to regularize the selection process to ensure that the distribution of selected responses does not drift too far from the distribution of responses seen during reward model training in a strong information theoretic sense.  The authors provide strong theoretical guarantees on the efficacy, monotonicity, and optimality of their approach, and demonstrate that it can be efficiently implemented in practice through a simple rejection sampling scheme \citep{block2023sample}.  However, this approach is overly conservative in practice, preventing the selection of genuinely better responses that are slightly out-of-distribution, and thus fails to fully leverage the benefits of inference-time scaling.  The problem arises because the regularization is \emph{distributional} and thus simultaneously ensures all possible OOD responses are penalized equally, regardless of whether or not they are likely to arise from imperfections in $\rhat$.  The starting point for this paper is thus to ask: \emph{\textbf{Can we design a BoN sampling scheme that is both robust to reward hacking and still able to leverage the full benefits of inference-time scaling?}}

\textbf{Contributions.} We answer this question in the affirmative by introducing \emph{caution}, an inference-time instantiation of the \emph{pessimism} principle from Reinforcement Learning (RL) \citep{jin2021pessimism,guo2022model}. Pessimism relies on lower confidence bounds to avoid selecting OOD actions with uncertain rewards. While \citet{huang2025best} implements pessimism at the \emph{distribution} level by constraining the sampling distribution to remain close to the base policy, we instead apply it at the \emph{reward} level: we penalize OOD responses by subtracting per-response uncertainty estimates from the scores assigned by $\rhat$, and then select the response with the highest pessimistic score.

Conceptually, caution is the dual of \emph{curiosity}, a principle used to drive optimistic exploration in deep RL \citep{pathak2017curiosity,burda2018exploration}. As in curiosity, we measure uncertainty by fixing a simple learning target (e.g., a frozen feature embedding), training a student model to predict this target on in-distribution data, and using the prediction error as an uncertainty signal. Unlike curiosity, however, our setting is fully \emph{offline}, so no continual student training is required---making the method significantly simpler and more practical.  
Through extensive empirical evaluation, we show that caution is computationally efficient and effectively mitigates reward hacking in BoN sampling. We further provide a theoretical analysis in a simplified linear setting, proving that caution strictly improves over standard BoN. Taken together, these results demonstrate that caution is both a powerful practical solution to reward hacking and compelling evidence for the broader efficacy of curiosity-style uncertainty signals in OOD detection and pessimistic policy learning for language models. \looseness=-1

%% file: sec3-methodology.tex
\section{Methodology}
\begin{figure}[t]
    \centering
    \includegraphics[width=0.85\columnwidth]{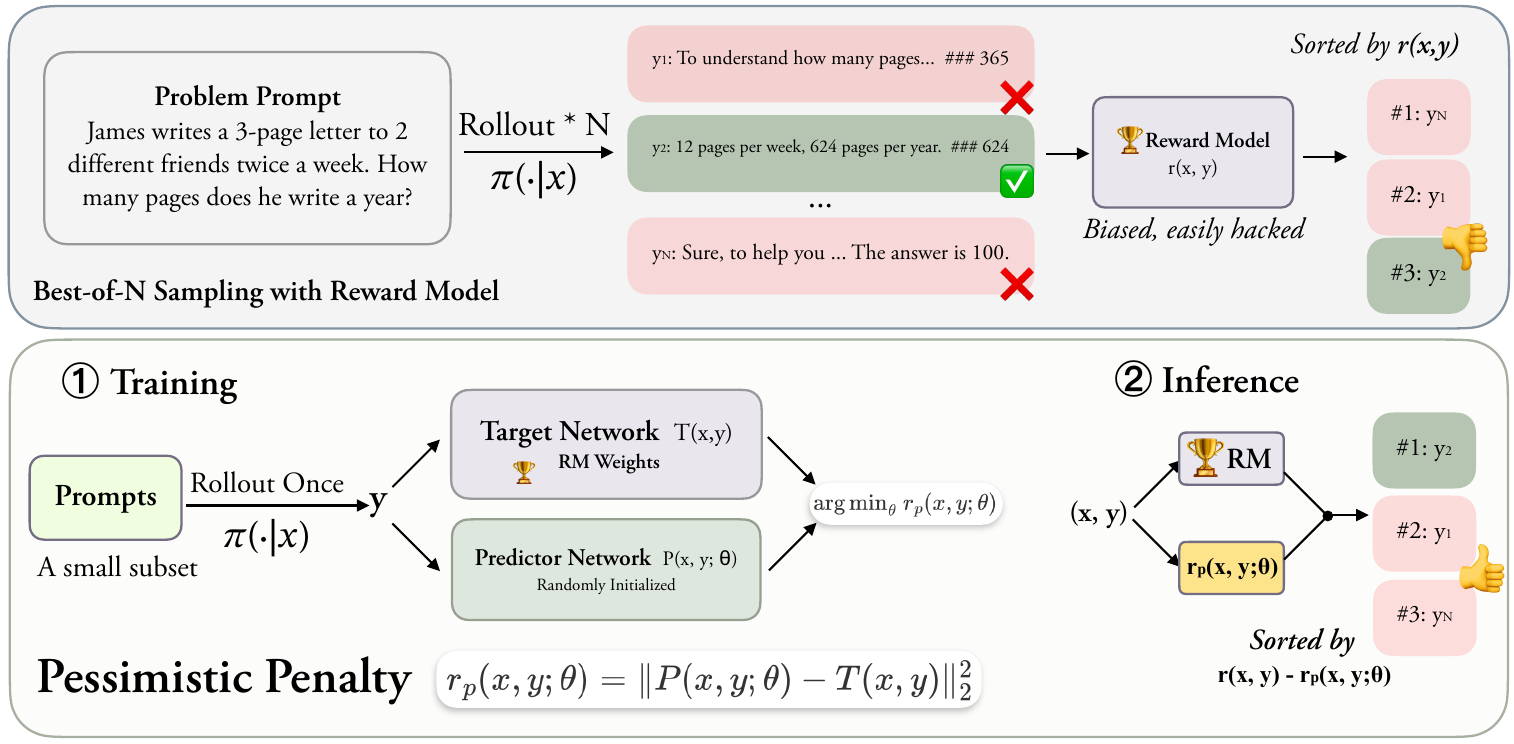}
    \vspace{-0.25cm}
    \caption{\textbf{Overview.} Predictor is trained to match RM features on typical responses; at inference, we select the candidate with the highest pessimistic reward, down‑weighting OOD ones.}
    \vspace{-0.25cm}
    \label{fig:overview}
\end{figure}

\subsection{Problem Formulation}

Consider a trained language model (LM) $\pi$, where $\pi: \cX \to \Delta(\cY)$ is a map from prompts $x \in \cX$ to distributions over responses $y \in \cY$.  While the LM is typically an autoregressive model, generating one token at a time before feeding the generated tokens back in as context, we abstract this process away and instead consider prompts and responses as single entities, possibly consisting of the concatenation of many individual tokens.  Suppose we have access to a learned reward model $\rhat: \cX \times \cY \to \rr$ that assigns a score to a prompt-response pair $(x,y)$ intended to reflect the quality of the response $y$ to the prompt $x$.  While the learner has access to $\rhat$, the true goal is to output a response that approximately maximizes some ground-truth reward $\rstar: \cX \times \cY \to \rr$, which is unknown to the learner. We focus on \emph{reward-hacking} behavior that seeks to maximize $\rhat$ at the expense of $\rstar$. 

Given a prompt $x$, the learner samples $N$ candidate responses $y_1, \dots, y_N \sim \pi(\cdot | x)$ independently from the LM and then selects one of these responses $y_{\ihat}$ with $\ihat \in [N]$ to output.  Perhaps the simplest strategy is the \emph{Best-of-$N$} (BoN) strategy, which simply selects the response with the highest reward model score: $\ihat = \ihat_N = \argmax_{i \in [N]} \rhat(x, y_i)$. Unfortunately, this strategy is vulnerable to reward hacking when we only have $\rhat \approx \rstar$ on typical samples from $\pi$. As $N$ grows large, the distribution of $\ihat_N$ can differ substantially from that of typical samples from $\pi$, leading to poor performance with respect to $\rstar$ as shown in \citep{huang2025best}.  While that work proposes to address this issue by constraining the BoN sampling distribution $\pihat$ to be close to $\pi$, we instead propose to directly apply regularization to the reward estimates themselves.  Thus, we will design an uncertainty estimate $\alpha: \cX \times \cY \to \rr_{\geq 0}$ that quantifies how uncertain we are about the reward model's estimate $\rhat(x,y)$ for a given prompt-response pair $(x,y)$ and then define $\rlcb(x, y) = \rhat(x,y) - \lambda \alpha(x,y)$ for some $\lambda > 0$ to be a pessimistic variant of the reward model that penalizes uncertain responses \citep{jin2021pessimism}.  As long as $\alpha(x,y)$ is small for typical samples from $\pi(\cdot | x)$ and large for `atypical' samples, this approach will successfully penalize reward-hacking responses by ensuring that $\rlcb \leq \rstar$ and this inequality is approximately tight for actually good quality responses.  The final algorithm will thus be to return $\ihat = \argmax_{i \in [N]} \rlcb(x, y_i)$ and the question becomes \emph{what choice of uncertainty estimate $\alpha$ appropriately captures OOD responses in a reward-aware manner?}

\subsection{Curiosity to Caution: Instantiating Pessimism}

Our approach introduces the principle of \emph{caution}, a phenomenon dual to the well-known technique of \emph{curiosity} used in online RL to incentivize exploration \citep{pathak2017curiosity,burda2018exploration}.  While in online RL, OOD states are desirable (as they represent good exploration targets), in our offline setting, OOD responses are to be avoided, as we have no reliable way to estimate their true reward given that $\rhat \approx \rstar$ only on typical responses from $\pi$.  We draw inspiration from  curiosity \citep{pathak2017curiosity} and Random Network Distillation (RND) \citep{burda2018exploration} in order to ensure that selected responses remain close to the distribution $\pi$ on which $\rhat$ is reliable.

In curiosity and RND, the core idea is to continually train a predictor network to match the outputs of a fixed target that is easily evaluated by the learner; the supervised learning error then becomes a proxy for OOD detection, with the intuition being that the predictor will be accurate on states similar to those seen during training and inaccurate otherwise.  While empirically successful as an `intrinsic reward' for optimistic exploration, this method can be challenging to implement due to the continued training of the predictor during online RL, which can lead to nontrivial memory and time overheads in practical RL pipelines.  Our method uses the same core idea of using the supervised learning error as a proxy for OOD detection, but is \emph{trained  fully offline}, which greatly increases practicality and ease of implementation.

\paragraph{Implementing ``Caution''} We employ two neural networks operating on the internal representations of the reward model $R(x,y)$: (1) a fixed \emph{target network} $T(x,y) = h_L^R(x,y)$, defined as the hidden state extracted from layer $L$ of the frozen reward model, and (2) a trainable \emph{predictor network} $P_\theta(x,y)$ that learns to predict this hidden state. The first $L$ layers of the reward model serve as our target feature extractor, with their output $h_L^R(x,y)$ as the prediction target. The predictor network $P_\theta(x, y)$ is a separate trainable network with parameters $\theta$ that can use various architectural choices (shared embeddings, simplified encoders, projection layers) but is designed to be lightweight to ensure efficient training and inference.

\begin{figure}[t]
    \centering
    \includegraphics[width=\columnwidth]{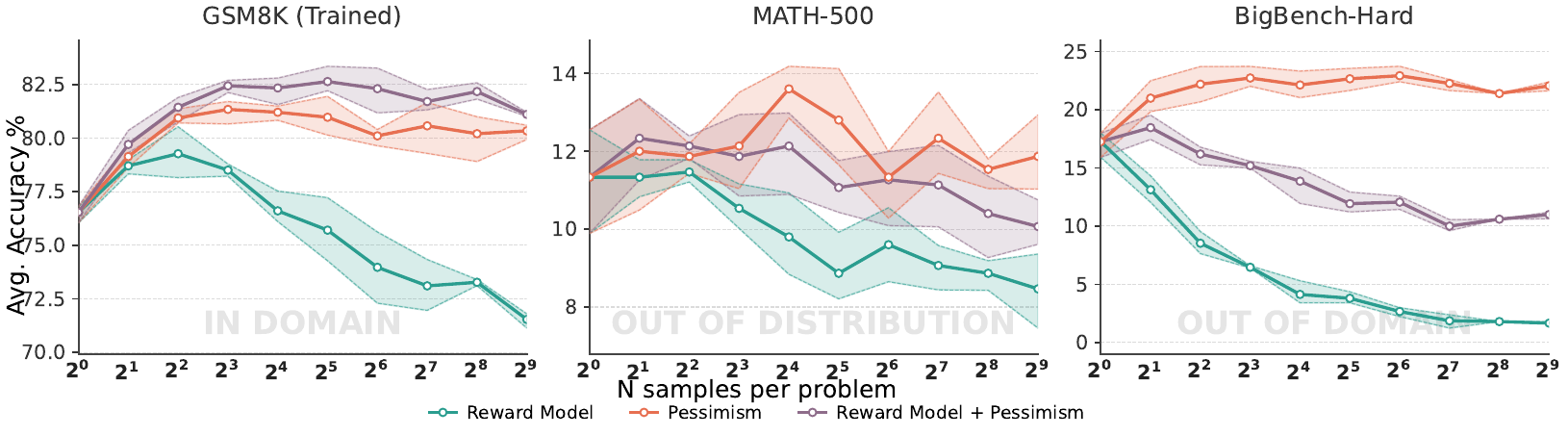}

\caption{\textbf{Scaling over $N$ across distributions and domains.} Best-of-$N$ sampling on GSM8K, MATH-500, and BigBench-Hard. Curves compare selection by Reward Model, Pessimism, and Reward Model + Pessimism. Note that the \emph{pessimism function is trained only on GSM8K}; thus, MATH-500 represents an out-of-distribution setting, while BigBench-Hard represents a fully out-of-domain setting.}\label{fig:tripanel-scaling}
\end{figure}

\paragraph{Training Process.}  We train our predictor $P_\theta$ on a dataset $\mathcal{D}_{\text{train}}$ of prompt-response pairs constructed directly from the benchmark train split using mean squared error loss: $\mathcal{L}(\theta) = \mathbb{E}_{(x,y) \sim \mathcal{D}_{\text{train}}} \left[ \|P_\theta(x, y) - T(x, y)\|^2 \right]$,
where the norm is Euclidean.  We construct responses $y \sim \pi(\cdot | x)$ by sampling from $\pi$ as an efficient empirical proxy for the training distribution of the reward model.  By training on this supervision-free distribution, the predictor achieves low error on in-distribution patterns while preserving high error on distributional outliers.  We emphasize that the data requirements for this step are restricted to \emph{prompts}, which are often much cheaper to collect than the high-quality \emph{labelled} data used to train reward models; a gold standard would be to reuse the same prompts used to train the reward model itself, but in practice we find that our method's OOD detection is relatively robust across different prompt distributions (cf. \Cref{ssec:main-empirical-results}).

\paragraph{Pessimistic Reward Estimation at Inference Time.}  Once we have trained our predictor $P_\theta$, we can use it to define our uncertainty estimate $\alpha(x,y) = \|P_\theta(x,y) - T(x,y)\|^2$, the prediction error of the predictor on the target network's features.    We then plug this into our pessimistic reward estimate $\rlcb(x,y) = \rhat(x,y) - \lambda \alpha(x,y) = \rhat(x, y) - \lambda \cdot \norm{P_\theta(x, y) - T(x, y)}^2$.  Note that this score is very easy to compute at inference time, requiring only two forward passes (one through $P_\theta$ and one through $T$) in addition to the forward pass through the reward model to compute $\rhat$; note that all of these passes can be fully parallelized and the cost of evaluating $\alpha(x,y)$ is on the same order as that of evaluating $\rhat$ due to the reuse of features from $\rhat$.  The parameter $\lambda$ controls the strength of the pessimism penalty, with $\lambda = 0$ recovering standard BoN sampling.  The final selection becomes $\ihat = \argmax_{i \in [N]} \rlcb(x, y_i)$, i.e., choosing the response with the highest pessimistic reward estimate. This procedure is summarized in \Cref{fig:overview}.

\subsection{Theoretical Interpretation}\label{ssec:theory}

In order to provide further motivation for our approach, we analyze in \Cref{app:proofs} a simple theoretical setting in which our approach provably improves upon BoN.  While we defer the details to the appendix, we summarize the main theorem here:
\begin{theorem}[Informal version of \Cref{thm:main}]\label{thm:informal}
    Let $y_1, \dots, y_N \in \rr^d$ be i.i.d. samples from a model $\pi$ and let $\rstar(y)$ be a linear reward function.  Let $\istar = \argmax_{i \in [N]} \rstar(y_i)$ be the optimal response and let $\ihat = \argmax_{i \in [N]} \rhat(y_i)$ be the response selected by BoN using a learned reward model $\rhat$.  If $\ipess = \argmax_{i \in [N]} \rlcb(y_i)$, where $\rlcb$ is our caution-regularized reward estimate, then under suitable conditions on the target network, $\rstar$, $\rhat$, $\pi$, and the predictor network, it holds that
    \begin{align}
        \ee\left[ \rstar(y_{\ipess}) - \rstar(y_{\ihat}) \right] \gtrsim \sqrt{\log(N)}, \qquad \text{and} \qquad \lim_{N \uparrow \infty } \frac{\ee\left[ \rstar(y_{\istar}) - \rstar(y_{\ipess}) \right]}{\ee\left[ \rstar(y_{\istar}) \right]} = 0.
    \end{align}
\end{theorem}
While the assumptions and conditions of \Cref{thm:informal} are necessarily somewhat strong in order to facilitate the analysis, the theorem provides a proof-of-concept that our approach can provably outperform BoN in a stylized setting as well as, to the best of our knowledge, the first theoretical guarantee on the success of curiosity- and RND-style methods for OOD detection.

%% file: sec4-experiments.tex
\section{Experiments}

\input{tables/scaling_results}
We now empirically evaluate our proposed approach, with a focus on answering the following three questions: (1) How well does caution mitigate reward hacking as the number of candidates $N$ increases? (2) What design decisions (e.g. architecture and training hyperparameters) contribute most to the overall success of our method? (3) In which scenarios does our method outperform standard BoN sampling, and what factors drive its success?  We now briefly describe our empirical setup (with full details deferred to \Cref{app:exp-details}) before addressing each of these questions in turn. 

\paragraph{Experimental Setup.} We use \texttt{Llama-3.2-3B-Instruct} \citep{dubey2024llama} as our language model $\pi$ due to its strong reasoning capabilities and open-source availability.\footnote{Some other open-weight models are thought to be contaminated with benchmark datasets \citep{wu2025reasoning}, which could skew our results, further motivating our choice in model.}  We consider three prompt distributions coming from reasoning datasets of varying difficulty: \texttt{GSM8K} \citep{cobbe2021training}, \texttt{MATH-500} \citep{hendrycks2021measuring}, and \texttt{BigBench-Hard} \citep{suzgun2022challenging}.  Our ground truth reward $\rstar$ is binary, with reward given if and only if a response provides the correct answer to the given prompt.  While we consider several reward models $\rhat$, we focus primarily on \texttt{OASST} \citep{kopf2023openassistant}, which \Cref{fig:Fig-1} demonstrates provides better performance than comparably sized reward models like \texttt{Skywork-Reward-V2-Qwen3-0.6B} and \texttt{Skywork-Reward-V2-Llama-3.2-1B} \citep{liu2025skywork}, making it a strong baseline for demonstrating reward hacking phenomena.  For each prompt and $N$, we use vLLM \citep{kwon2023efficient} to generate $N$ independent responses from the base model $\pi$ (and boostrap this process 3 times to generate confidence intervals), before scoring each response with $\rhat$ and measuring the performance of the selected response according to $\rstar$.  To train our predictor network $P_\theta$, we use an independent dataset of responses generated by the base model $\pi$ itself on the training split of \texttt{GSM8K}, and evaluate the performance of the uncertainty estimates both in-distribution (prompts from the test set of \texttt{GSM8K}) and out-of-distribution (prompts from other reasoning datasets).  In order to maintain consistency, we normalize all rewards $\rhat$ to be centred with unit variance using an independent set of responses and we do the same for the uncertainty estimates from $P_\theta$.  \looseness=-1

\subsection{Caution Mitigates Reward Hacking}\label{ssec:main-empirical-results}

\input{tables/weight_ablation}

Our main results demonstrate that our proposed mechanism of caution \textbf{mitigates reward-hacking and leads to improved performance of BoN sampling as $N$ increases both for in- and out-of-distribution prompts.} In \Cref{fig:Fig-1}, we exhibit the performance of BoN sampling on \texttt{GSM8K} for several choices of reward models and observe that they all exhibit reward hacking, with performance degrading for large $N$.  We also see that using our pessimistic approach, this reward-hacking is substantially mitigated, with performance improving monotonically with $N$ for all considered $N$.  Moreover, we get a substantial performance boost of $4.2\%$ over peak accuracy for the reward model alone and an astonishing $15.5\%$ boost over the final accuracy of the reward model alone.  This trend compares favorably with that of \citet{huang2025best}, who observe monotonicity of performance in $N$, but struggle to outperform BoN sampling for optimally tuned $N$ for most choices of $\pi$ and $\rhat$.

The above results involve evaluating caution on the same distribution of prompts used to train the predictor network $P_\theta$, which is likely beneficial to its performance.  In reality, we expect to use caution in scenarios where the prompts are potentially OOD and hope that the uncertainty estimates remain valid.  To evaluate the extent to which this holds, we used the same uncertainty estimates to produce pessimistic BoN sampling on two significantly harder reasoning datasets: \texttt{MATH-500} and \texttt{BigBench-Hard}.  These datasets comprise prompts that are significantly different from those in \texttt{GSM8K}, with the latter even coming from a different, non-mathematical domain.

We report the results of these OOD prompt experiments in \Cref{fig:tripanel-scaling} and \Cref{tab:scaling_results}, comparing the performance of the in-distribution prompts (left) with the two other tasks (centre and right).  As in \Cref{fig:Fig-1}, we sweep over a logarithmic grid in $N$ from $1$ to $512$ and compare three approaches: BoN using only $\rhat$, BoN using only the uncertainty estimates from $P_\theta$, and BoN using caution.  As in \Cref{fig:Fig-1}, we see that caution successfully mitigates reward hacking while preserving the benefits of reward-guided selection.  The results on \texttt{MATH-500} reveal a different pattern: pessimism-only outperforms the combined approach, achieving 13.6\% peak accuracy with only 1.7 points degradation compared to 2.3 points for the combined method. This counterintuitive finding reflects the increased problem difficulty—while \texttt{MATH-500} shares the mathematical reasoning domain with \texttt{GSM8K}, the problems are significantly more complex and multi-step, making the task of $\rhat$ more difficult.   In this regime, $\rhat$ becomes less reliable at distinguishing genuine quality improvements, making its contribution less beneficial. However, caution still effectively prevents the severe degradation seen with BoN (3.0 points).  The most challenging task we consider is \texttt{BigBench-Hard}, where the reward model $\rhat$ fails catastrophically.  Indeed, we observe the counterintuitive phenomenon that using the uncertainty penalty alone is the most performant of all three methods, doing even better than combining uncertainty with $\rhat$.  This surprising result is due to the fact that when reward models encounter problems substantially harder than their training distribution, their learned spurious correlations dominate their judgments, leading to systematic selection of responses that superficially mimic quality patterns without genuine correctness.  Caution, while also operating outside its training distribution, avoids being misled by these hacking features and maintains stable performance by preferring responses that match familiar distributional patterns rather than chasing unreliable reward signals.\looseness=-1

\subsection{Ablation Studies}\label{ssec:ablations}
\input{tables/ablations}

In this section, we dissect the design and implementation of our proposed caution mechanism to understand which components contribute most to its success.  
Further ablation studies, including detailed inspections of uncertainty scores for individual responses and the effect they have on $\rhat$ and $\rstar$ scores can be found in \Cref{app:ablation}, while additional details on the ablations can be found in \Cref{app:exp-details}. \looseness=-1

We evaluate three critical design decisions: (1) the precise architectural choice of the predictor network; (2) the use of pre-trained embeddings in the target network; (3) the use of a projector layer between the predictor and target networks.  We also conduct two additional ablations to validate our core design choices: (a) the choice of mixing weights between $\rhat$ and the uncertainty estimates; (b) the choice of distilling reward model features versus random network features.

\paragraph{Architectural Ablations.}  We first examine the extent to which the complexity of $P_\theta$ and its relationship to the target network $T$ affects performance, in particular whether $P_\theta$ needs to share an architecture with the teacher $T$. The results are shown in \Cref{tab:rnd-ablation}: while greater flexibility in $P_\theta$ leads to a better reconstruction loss, this does not transfer to better performance when used to instantiate caution.  This occurs because improving the reconstruction loss can actually hurt the ability of $P_\theta$ to detect distributional novelty, as it becomes too good at reconstructing even out-of-distribution features.    We observe a similar trend when considering whether or not to tune the embeddings or share them between $P_\theta$ and $T$: while sharing embeddings leads to better reconstruction, it can hurt the final performance, which suggests that forcing predictors to learn representations from scratch creates more sensitive distributional boundaries than inheriting potentially overly simplified features from pre-trained models.  Finally, we see that adding a projection layer between the predictor and target networks provides consistent but modest benefits, indicating that information bottlenecks can help prevent pure memorization while preserving signal quality.  Together, these findings establish a key empirical finding for instantiating caution: \textbf{effective detection requires a careful balance between reconstruction accuracy and novelty sensitivity.}

\paragraph{Strength of regularization with Caution.}  To validate our core design choices, we systematically vary the mixing weights between the uncertainty score and that of $\rhat$ to identify the optimal balance for reward hacking mitigation and understand the robustness of our results to this choice.  Note that due to the normalization of $\rhat$ and uncertainty estimates described above, the strength $\lambda$ can be viewed as a direct measurement of the influence of the pessimism relative to $\rhat$.  We report our findings in \Cref{tab:weight_ablation} and \autoref{app:ablation}.  We observe that our approach is relatively robust to the choice of $\lambda$, with moderate to high weights (0.6-0.8) achieving optimal performance.  While the precise optimal weight will vary by task and choice of $\rhat$, with larger $\lambda$ being required in situations where $\rhat$ is less reliable, the results suggest that our approach does not require delicate tuning to be effective.

\paragraph{Comparing Caution to RND.}  Finally, we compare our approach of distilling reward model features (motivated by curiosity in \citet{pathak2017curiosity}) against the RND \citep{burda2018exploration}, which takes the teacher $T$ to be a randomly initialized network (possibly on top of pre-trained embeddings), testing whether our hypothesis about distributional familiarity requires semantic grounding in the reward model's representations.  We again consider a range of mixing weights $\lambda$ and report the results in \Cref{tab:weight_ablation}.  The results reveal a stark contrast between caution and RND, with the latter exhibiting dramatically inferior performance across all choices of $\lambda$.  This observation provides evidence for the hypothesis that effective distributional regularization requires semantic grounding: randomly initialized features cannot provide meaningful distributional boundaries, while reward model features capture task-relevant patterns that enable robust novelty detection.  The results establish that curiosity-driven pessimism succeeds not merely because of prediction error signals, but specifically because these signals are computed relative to the reward model's learned task representations.  

\subsection{Why It Works: A Case Study}

\begin{figure*}[t]
    \centering
    \includegraphics[width=\columnwidth]{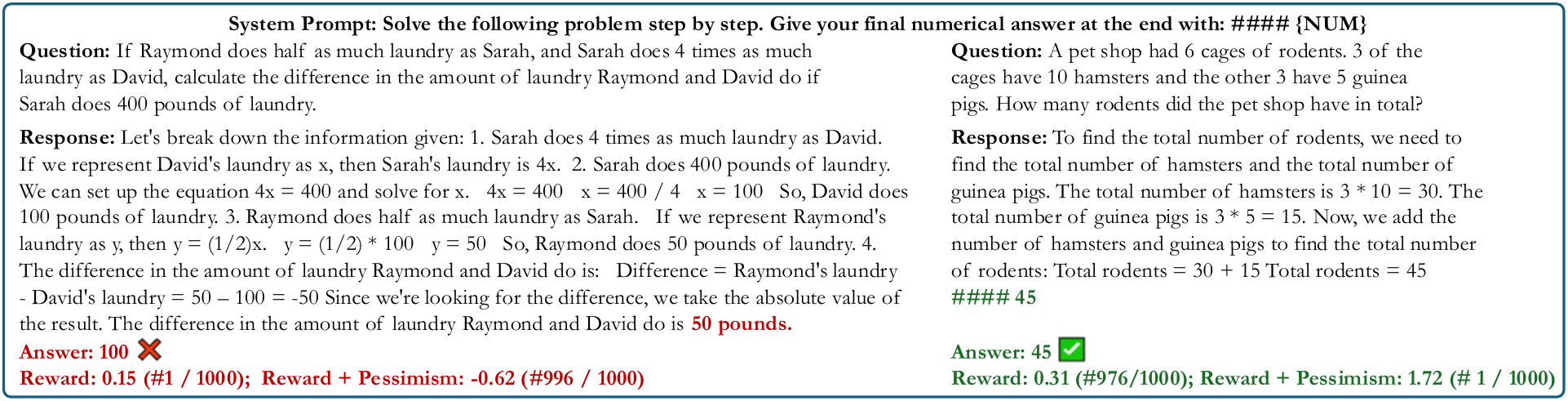}
    \caption{\textbf{Contrasting Selection Behaviors: Reward Hacking vs. Format Compliance.} Two representative examples showing how reward models favor verbose responses regardless of correctness, while our curiosity-driven pessimism prioritizes format compliance and distributional familiarity. RM assigns high scores to detailed responses regardless of correctness, while pessimism detects distributional deviation from training patterns and prefers correctly formatted solutions.}
    \label{fig:case_study_examples}
\end{figure*}

We now turn to a case study to illustrate the mechanism by which caution mitigates reward hacking and improves performance.  In \Cref{fig:case_study_examples}, we present two representative examples of questions-response pairs.  On the left, we see an instance that $\rhat$ scores highly (99.8th percentile of all scored responses) despite being incorrect; this response is verbose and contains multiple mathematical reasoning steps, but ultimately fails to provide the correct answer and does not follow the required formatting, instead demonstrating exactly the type of reward hacking our method targets: superficial mimicry of quality patterns (detailed explanations, step-by-step structure) without genuine correctness or adherence to specified requirements.  Note that the caution score for this response appropriately identifies it as OOD, likely due to its failure to follow the formatting requirements that are common in high-quality responses in the training data.  

By contrast, on the right of \Cref{fig:case_study_examples}, we see a concise, accurate response following the formatting requirements.  While $\rhat$ does not rank this response as well as the first, the pessimism recognizes it as familiar, matching the patterns of high-quality responses in the training data that consistently follow formatting specifications and provide direct, correct solutions.  This juxtaposition reveals that reward models can conflate verbosity with quality, particularly at larger $N$ where more elaborate responses become available. 
Our method effectively distinguishes between responses that \emph{appear} sophisticated (verbose explanations) and those that \emph{are} actually correct and compliant with task specifications. \looseness=-1

%% file: tables/scaling_results.tex
\begin{table*}[t]
\centering
\caption{\textbf{Scaling Performance Across Datasets with Reward Hacking Mitigation.} \textbf{Peak Acc}: highest accuracy achieved across all $N$ values. \textbf{Final Acc}: accuracy at $N=512$ where reward hacking is most severe. \textbf{Degradation}: performance drop from peak to final (lower is better). Results show mean with 95\% confidence intervals as subscripts across 3 bootstrap runs. Best results are in \textbf{bold} and second-best are \underline{underlined}.}
\label{tab:scaling_results}
\footnotesize
\setlength{\tabcolsep}{3pt}
\newlength{\datasepspace}
\setlength{\datasepspace}{12pt}
\renewcommand{\arraystretch}{1.0}
\begin{tabular}{@{}l*{9}{c}@{}}
\toprule
\multirow{2}{*}{\textbf{Method}} & \multicolumn{3}{c}{\textbf{GSM8K}~\citeyearpar{cobbe2021training}} & \multicolumn{3}{c}{\textbf{MATH-500}~\citeyearpar{lightman2023let}} & \multicolumn{3}{c}{\textbf{BBH-Hard}~\citeyearpar{suzgun2022challenging}} \\
\cmidrule(l{4pt}r{8pt}){2-4}\cmidrule(l{4pt}r{8pt}){5-7}\cmidrule(l{4pt}){8-10}
 & Peak & Final & Degr. & Peak & Final & Degr. & Peak & Final & Degr. \\ 
\midrule
Reward Model & 79.3$_{\pm1.2}$ & 71.5$_{\pm0.3}$ & 7.7 & 11.5$_{\pm0.3}$ & 8.5$_{\pm1.0}$ & 3.0 & \underline{17.3$_{\pm1.0}$} & 1.7$_{\pm0.1}$ & 15.6 \\
Pessimism Only & \underline{81.3$_{\pm0.5}$} & \underline{80.3$_{\pm0.3}$} & \textbf{1.0} & \textbf{13.6$_{\pm0.7}$} & \textbf{11.9$_{\pm1.0}$} & \textbf{1.7} & \textbf{22.9$_{\pm0.7}$} & \textbf{22.1$_{\pm0.4}$} & \textbf{0.9} \\
RM + Pessimism & \textbf{82.6$_{\pm0.6}$} & \textbf{81.1$_{\pm0.1}$} & \underline{1.5} & \underline{12.3$_{\pm1.0}$} & \underline{10.1$_{\pm0.6}$} & \underline{2.3} & \underline{18.5$_{\pm1.0}$} & \underline{11.0$_{\pm0.3}$} & \underline{7.5} \\
\bottomrule
\end{tabular}
\end{table*}

%% file: tables/weight_ablation.tex
\begin{table}[t]
\centering
\caption{\textbf{Weight Ablation and Design Comparison on GSM8K.} We compare different mixing weights between uncertainty score and $\rhat$ score, contrasting our approach (distilling reward model features) against traditional RND (distilling random network features). \textbf{Peak Acc}: best accuracy across all $N$. \textbf{Final Acc}: accuracy at $N=512$. Results demonstrate that distilling reward model representations substantially outperforms random network distillation across weight combinations.}
\label{tab:weight_ablation}
\small
\setlength{\tabcolsep}{6pt}
\renewcommand{\arraystretch}{1.1}
\begin{tabular}{@{}lcccc@{}}
\toprule
\multirow{2}{*}{\textbf{Pessimism strength} ($\lambda$)} & \multicolumn{2}{c}{\textbf{Caution (Ours)}} & \multicolumn{2}{c}{\textbf{Traditional RND}} \\
\cmidrule(l{8pt}r{8pt}){2-3}\cmidrule(l{8pt}){4-5}
 & Peak Acc & Final Acc & Peak Acc & Final Acc \\
\midrule
0.0 (RM Only) & 78.9 & 71.2 & 78.9 & 71.2 \\
0.2 & 79.5 & 72.5 & 78.7 & 71.3 \\
0.4 & 80.3 & 76.0 & 78.5 & 72.1 \\
0.6 & \underline{81.5} & \underline{80.2} & 78.4 & 73.1 \\
0.8 & \textbf{82.1} & \textbf{81.0} & 78.1 & 74.9 \\
1.0 (Caution Only) & 81.3 & 79.8 & 77.0 & 72.9 \\
\bottomrule
\end{tabular}
\end{table}

%% file: tables/ablations.tex
\begin{table}[t]
\centering
\small
\setlength{\tabcolsep}{8pt}
\caption{Ablation study of predictor architecture. \textbf{Training Loss} measures how well the predictor fits reward model representations. \textbf{Peak Acc} shows best performance across all $N$. \textbf{Final Acc} shows performance at largest $N$, where reward hacking is most severe.}
\begin{tabular}{lccc}
\toprule
\textbf{Predictor Configuration} & \textbf{Reconstruction Loss} & \textbf{Peak Acc (\%)} & \textbf{Final Acc (\%)} \\
\midrule
\multicolumn{4}{l}{\textit{Simplified Architecture Variants}} \\
\quad Lightweight + Trainable Emb. & 0.242 & 82.2 & 82.2 \\
\quad Lightweight + Frozen Emb. & 0.246 & 81.5 & 81.3 \\
\quad Lightweight + Separate Emb. & 0.255 & 82.7 & 81.8 \\
\quad Lightweight w/o Projection & 0.281 & 81.8 & 81.3 \\
\midrule
\multicolumn{4}{l}{\textit{Full Architecture Variants}} \\
\quad Full + Trainable Emb. & 0.127 & 80.3 & 80.2 \\
\quad Full + Frozen Emb. & 0.127 & 80.4 & 78.4 \\
\bottomrule
\end{tabular}
\label{tab:rnd-ablation}
\end{table}

%% file: sec5-discussion.tex
\section{Discussion}
In this work, we investigated instantiating pessimistic reward estimation with the principle of \emph{caution}, an approach adapted from curiosity-driven exploration in RL  that uses a supervised learning error as a measure of distributional uncertainty and penalizes the estimated reward of uncertain, out-of-distribution (OOD) inputs.  While our focus was on mitigating reward hacking in Best-of-$N$ sampling, a particular inference-time scaling technique, our results demonstrate that, when properly applied on top of pre-trained features, caution can effectively detect OOD samples and instantiate pessimistic policies in general. While we defer a more complete survey of related work to \Cref{app:related}, we now provide a brief summary thereof as well as discuss future directions.

\paragraph{Related Work.} Best-of-$N$ sampling was introduced in \citet{stiennon2020learning} and has been empirically investigated in many reasoning settings \citep{cobbe2021training,lightman2023let,li2022competition,brown2024large}.  While often effective, when the scoring rule is learned (as opposed to oracular), the approach has long been shown to be vulnerable to reward-hacking \citep{amodei2016concrete,skalse2022defining,gao2023scaling}.  While many attempted mitigations of reward-hacking have been explored for RL \emph{finetuning}, relatively few works have focused on the \emph{inference-time setting}.  Of particular note is \citet{huang2025best}, which proposes a distributional regularization approach that samples according to a $\chi^2$-regularized BoN procedure.  While theoretically well-motivated and empirically effective at ensuring monotonicity in $N$, the approach is overly conservative in practice.  Another approach is that of \citet{jinnai2024regularized}, who apply distributional regularization with respect to a Wasserstein distance on some embedding space; while potentially effective, the computation required grows quadratically in $N$, making it impractical for even moderate values of $N$.  In contradistinction to these works, our approach applies pessimism directly to the estimated rewards in a way that naturally leverages the beyond-worst-case errors present in estimated reward models in a way that is impossible for these distributional regularization approaches.  

Our key technique of \emph{caution} is built on top of the foundational curiosity \citep{pathak2017curiosity} and Random Network Distillation (RND) \citep{burda2018exploration} from classical deep RL.  While the such techniques have been very popular in aiding exploration, the extent to which they can be used for OOD detection and pessimistic learning has been a matter of some debate, with \citet{rezaeifar2022offline} claiming negative results and \citet{ciosek2019conservative,nikulin2023anti} demonstrating some positive results.  While these works are (i) in classical RL or supervised learning settings and (ii) do not use pre-trained features, our results morally align with the latter camp, demonstrating that such approaches can be effective for OOD detection and pessimistic learning in language models.

\paragraph{Future Directions.}  Our work provides strong evidence that caution, when correctly applied on top of pre-trained features, can be an effective detector of OOD text.  While we instantiate this approach for pessimistic reward estimation, it is natural to wonder if curiosity can be used as an explicit reward signal to encourage exploration of novel behaviors either purely during inference or during RL post-training of reasoning models.  While some preliminary work like \citet{gao2025navigate} has explored this idea, we believe that their mixed results stem from the fact that the curiosity module was trained from scratch rather than on top of pre-trained features, which we find to be so effective for OOD detection.  Another interesting direction is to explore the extent to which our proposed approach can help ensure continued AI alignment in the face of adversarial prompting or distributional shift due to inference-time scaling.  While the tasks we consider in this work are related to reasoning, we expect that our results would carry over to safety and alignment tasks \emph{mutatis mutandis}, which would represent a promising new approach to ensuring robust alignment of language models.

%% file: sec2-related.tex
\section{Additional Related Work}\label{app:related}
We now give a more detailed discussion of how the present paper relates to prior work.  Our approach is situated at the intersection of three key areas of research: Best-of-$N$ sampling with reward models, reward hacking and its mitigation attempts, and curiosity-driven exploration techniques.

\paragraph{Best-of-N Sampling and Reward Models.}  Best-of-N (BoN) sampling generates $N$ candidate responses and selects the highest-scoring according to a reward model: $y^* = \arg\max_{y \in \mathcal{Y}_N} r(x, y)$~\citep{stiennon2020learning}. This technique has proven effective for mathematical reasoning~\citep{cobbe2021training, lightman2023let} and competitive programming~\citep{li2022competition} and the version of BoN with oracular rewards (called `pass@k') is a standard evaluation metric for reasoning models \citep{dubey2024llama,ouyang2022training,abdin2024phi,hui2024qwen2,lambert2024rewardbench}.  Reward models are typically trained on tuples  of preference data $(x, y_{\text{chosen}}, y_{\text{rejected}})$ consisting of a prompt $x$ as well as preferred and dispreferred responses $y_{\text{chosen}}$ and $y_{\text{rejected}}$ using ranking losses.  While a powerful paradigm,  this training procedure introduces vulnerabilities: models must infer complex reasons for preferences, often leading them to rely on spurious correlations like response length, formatting patterns, or stylistic preferences~\citep{liu2024rm}. High-quality reward models are also computationally expensive to train, requiring resources similar to base language models. This cost makes \emph{ensemble approaches}, a common strategy for reducing vulnerabilities in ML systems, impractical, as organizations typically train only a single reward model per domain~\citep{gao2023scaling, jinnai2024regularized}.  Our work focuses on mitigating reward hacking in BoN sampling with a single reward model, without retraining or modifying the reward model itself.

\paragraph{Reward Hacking and Mitigation Attempts.}  Reward hacking occurs when optimizing against imperfect proxy rewards leads to high-scoring but low-quality outcomes—a manifestation of Goodhart's Law~\citep{amodei2016concrete, skalse2022defining}. \citet{gao2023scaling} observed a common trend in BoN sampling, where as $N$ increases, performance first rises then falls.  As demonstrated theoretically in \citet{huang2025best}, these two phases correspond to an initial phase (when $N$ is small) and the learned reward $\rhat$ is an effective proxy for the true reward $\rstar$, leading to BoN succeeding, and a second phase where $N$ grows so large so as to produce atypical responses on which $\rhat$ is no longer effective due to spurious correlations and poor coverage during training \citep{eisenstein2023helping, liu2024rm, yu2025rewardanything}.  Broadly, there have been two main approaches to mitigating this reward-hacking problem: \emph{training-time} and \emph{inference-time} approaches.

\emph{Training-time ensemble methods} attempt to mitigate this by combining multiple reward models~\citep{coste2023reward, zhai2023uncertainty, rame2024warm, yan2024reward}. However, \citet{eisenstein2023helping} demonstrate that ensembles reduce but do not eliminate reward hacking. More critically, ensemble methods require multiple expensive reward sources, creating prohibitive costs for practical applications.

\emph{Inference-time approaches} work with existing reward models and intervene in the sampling procedure itself. Of note, \citet{jinnai2024regularized} proposed an approach that involves regularizing the BoN selection with respect to a Wasserstein distance in some embedding space and demonstrated empirical improvement over BoN as N grows.  Unfortunately, this approach requires computation to scale quadratically in $N$, making it somewhat impractical for larger $N$.  Of greatest relevance to the present work is that of \citet{huang2025best}, who propose instead to use an information-theoretic divergence, the $\chi^2$-divergence, to regularize the BoN selection.  They demonstrate that this approach is statistically and computationally optimal under the assumption that the learned reward model is close to the true reward in expected squared error loss for responses sampled from the base model.  While theoretically well-motivated and empirically effective at ensuring monotonicity in $N$, the approach is overly conservative in practice, leading to suboptimal performance when $N$ is moderate.  Another inference-time approach is that of \citet{khalaf2025inference}, who consider a KL-regularized BoN procedure as well as a computationally efficient approximation using the Poisson distribution.  Unfortunately, their method exhibits limited improvement over standard BoN, likely due to the fact that BoN is already effectively KL-regularized unless $N$ is growing exponentially.  Unlike these works, our approach does not seek to provide \emph{distributional} regularization but instead instantiates pessimism through the \emph{reward estimates} themselves.  Thus, our approach is able to leverage the fact that reward models may be imperfect in ways that are not information-theoretically `worst-case' and adapt accordingly.

\paragraph{Curiosity-Driven Exploration and Random Network Distillation.}  Curiosity-driven exploration in reinforcement learning uses prediction error to quantify novelty of given states \citep{pathak2017curiosity, sun2025curiosity, li2023curiosity}. The Intrinsic Curiosity Module (ICM)~\citep{pathak2017curiosity} uses prediction error $\|\hat{\phi}(s_{t+1}) - \phi(s_{t+1})\|_2^2$ as a curiosity signal, while Random Network Distillation (RND)~\citep{burda2018exploration} employs a simpler approach with two networks: a fixed target network $T(s)$ and a trainable predictor $P(s)$ that minimizes $\|P(s) - T(s)\|_2^2$.  In both cases, the supervised learning error serves as a proxy for OOD detection, with high error indicating states that are far outside the distribution (and thus worth exploring from the perspective of online RL).  While these methods were originally developed for exploration, they can instead be instantiated to prevent a policy from moving to far out of the distribution of observed states.  While there has been some debate as to the efficacy of these approaches for OOD detection and pessimistic learning \citep{rezaeifar2022offline, ciosek2019conservative, nikulin2023anti}, our results suggest that when properly implemented, such an approach can be effective in language modeling. 

%% file: app_proofs.tex
\section{Theory }\label{app:proofs}

In this section, we give formal statements for and prove our theoretical results.  We begin in \Cref{ssec:formal-setting} for formalizing the setting we consider in which pessimism can help mitigate reward hacking in Best-of-$N$ sampling.  We emphasize that this setting is a simplified abstraction intended to cleanly showcase the benefits of pessimism and help develop intuition for our approach and is not intended to represent a realistic model of language or reward modeling.  We then continue in \Cref{ssec:pessimism} to prove general results on the performance of BoN and pessimistic selection in our setting.  We proceed in \Cref{ssec:rnd} to introduce a simplified model of our proposed caution approach, motivated by Random Network Distillation (RND) \citep{burda2018exploration}.  We prove in simplified linear and two-layer ReLU settings that this approach can be used to instantiate pessimism under an idealized optimization model.  Finally, we combine these results in \Cref{ssec:main-result} to prove our main theorem that our proposed approach improves over BoN sampling in the model we consider.

\subsection{Formal Setting}\label{ssec:formal-setting}
In order to formalize the setting we consider, we suppose that a learner is given access to a language model $\piref$ that maps a prompt $x$ to a distribution over responses $y$. We will further identify these prompt-response pairs with their \emph{embeddings} in some linear feature space $\rr^d$ and, for the sake of simplicity, consider results on a per-prompt basis.  Thus we will assume a fixed prompt $x$ and write $\piref$ for $\piref(x)$.  We will further suppose that there exists a \emph{ground truth reward function} that is linear in the embedded features, i.e.,
\begin{align}
    \rstar(y) = \inprod{\thetastar}{y}. \label{eq:linear-reward}
\end{align}
Such an abstraction is partially justified by the fact that many modern reward models are linear layers on top of pre-trained Language Models \citep{lambert2024rewardbench,liu2024skywork,wang2024arithmetic,ArmoRM} and we thus simply directly associate the concatenation of prompt and response with its embedding in the final layer of the LM.

The goal of the learning problem is similar to that in \citet{huang2025best}: given access to $y_1, \dots, y_n \sim \piref$ sampled independently, as well as some imperfect proxy reward function $\rhat: \rr^d \to \rr$ (e.g., a learned reward model), we wish to select a response $y_{\ihat}$ with $\ihat \in [n]$ such that $\rstar(y_{\ihat})$ is as large as possible.  Note that if $\rhat$ were in fact a perfect reward model, so that $\rhat = \rstar$, then the optimal strategy would be to simply select $y_{\ihat}$ with $\ihat = \argmax_{i \in [n]} \rhat(y_i)$, which is the popular \emph{Best-of-$N$} (BoN) strategy.  However, as discussed in the main text, this strategy can fail when $\rhat$ is a poor approximation to $\rstar$ due to reward-hacking.

Unlike \citet{huang2025best}, wherein the authors suppose that $\rhat$ and $\rstar$ are close in expected squared error under $\piref$, we will instead suppose that $\rhat$ and $\rstar$ agree on a low dimensional subspace of $\rr^d$ but may be arbitrarily different on the orthogonal complement.  More precisely, suppose that there exists a linear subspace $V \subset \rr^d$ such that $\dim(V) = k \ll d$ and $\rstar(y) = \rhat(y)$ for all $y \in V$.  We will further make the assumption that $\rstar$ and $\rhat$ are linear functions, i.e., $\rhat(y) = \inprod{\thetahat}{y}$ and $\rstar(y) = \inprod{\thetastar}{y}$ for some $\thetahat,\thetastar \in \rr^d$.  To aid our analysis, we will make the simplifying assumption that $\thetastar \in V$, i.e. $\projperp \thetastar = 0$, representing the fact that the true reward is a function only of certain linear features; the error between the learned reward $\rhat$ and the groundtruth $\rstar$ is thus restricted to OOD effects, conceptualized as the orthogonal complement $V^{\perp}$ of $V$. 

Finally, in order to make the analysis tractable, we will assume that $\piref$ is a centred Gaussian distribution with covariance $\Sigma \in \rr^{d \times d}$, i.e., $\piref = \mathcal{N}(0, \Sigma)$.  Note that $\Sigma$ may have full rank $d$ and thus $\piref$ may have full support on $\rr^d$, which ensures that $\rhat \neq \rstar$ on $\vperp$, the orthogonal complement of $V$.

\subsection{Pessimistic Selection and Best-of-$N$ Sampling}\label{ssec:pessimism}
In this section we provide two key bounds on the performance of Best-of-$N$ sampling and pessimistic selection, which will demonstrate the benefits of pessimism in our setting.  Before doing so, we state and prove a simple bound on the best possible performance that can be achieved by any selection strategy.
\begin{proposition}\label{prop:optimal}
    Let $y_1, \dots, y_N \sim \piref$ be independent samples from $\piref$.  Then it holds that
    \begin{align}
        \ee\left[ \max_{i \in [N]} \rstar(y_i) | \thetastar \right] = \norm{\Sigma^{\nicefrac 12} \thetastar} \cdot M_N,
    \end{align}
    where $M_N = \ee\left[ \max_{i \in [N]} Z_i \right]$ and $Z_1, \dots, Z_N$ are i.i.d. $\cN(0,1)$ random variables.  Moreover, it holds that
    \begin{align}
        \sqrt{2 \log(N)} - o(1) \leq M_N \leq \sqrt{2 \log(N)}.
    \end{align}
\end{proposition}
\begin{proof}
    The second claim is a classical fact about the maximum of Gaussians \citep{wainwright2019high,vershynin2018high}, thus it suffices to prove the first statement.  Note that
    \begin{align}
        \rstar(y) = \inprod{\thetastar}{y} \sim \cN(0, \norm{\Sigma^{\nicefrac 12} \thetastar}^2).
    \end{align}
    Thus $\rstar(y) \stackrel{d}{=} \norm{\Sigma^{\nicefrac 12} \thetastar} Z$ for $Z \sim \cN(0,1)$.  The result then follows by positive homogeneity of the maximum and expectation.
\end{proof}

We now provide a lower bound on the performance of BoN in the setting we consider
\begin{proposition}\label{prop:bon}
    Let $y_1, \dots, y_N \sim \piref$ be independent samples from $\piref$ and let $\ihat = \argmax_{i \in [N]} \rhat(y_i)$ for $\rhat(y) = \inprod{\thetahat}{y}$.  Let $\istar = \argmax_{i \in [N]} \rstar(y_i)$ for $\rstar(y) = \inprod{\thetastar}{y}$.  Then it holds that
    \begin{align}
        \ee\left[ \rstar(y_{\istar}) - \rstar(y_{\ihat}) | \thetahat, \thetastar \right] &= \ee\left[ \rstar(y_{\istar}) \right] \left( 1 - \frac{1}{\sqrt{1 + \frac{\norm{\Sigma^{\nicefrac 12} \projperp \thetahat}^2}{\norm{\Sigma^{\nicefrac 12}\thetastar}}^2}} \right) \\
        &=  \left( 1 - \frac{1}{\sqrt{1 + \frac{\norm{\Sigma^{\nicefrac 12} \projperp \thetahat}^2}{\norm{\Sigma^{\nicefrac 12}\thetastar}}^2}} \right) \cdot \norm{\Sigma^{\nicefrac 12} \thetastar} \cdot M_N.
    \end{align}
\end{proposition}
\begin{proof}
    As we have already computed the expectation of $\rstar(y_{\istar})$ in \Cref{prop:optimal}, it suffices to compute the conditional expectation of $\rstar(y_{\ihat})$.  Note that because, conditioned on $\thetahat$ and $\thetastar$, $\rhat(y)$ and $\rstar(y)$ are jointly Gaussian, it holds that
    \begin{align}
        \ee\left[ \rstar(y) | \rhat(y) \right] &= \frac{\Cov(\rhat(y), \rstar(y))}{\Var(\rhat(y))} \cdot \rhat(y) \\
        &= \frac{\inprod{\thetastar}{\Sigma \thetahat}}{\norm{\Sigma^{\nicefrac 12} \thetahat}^2} \cdot \rhat(y) \\
        &= \frac{\norm{\Sigma^{\nicefrac 12} \thetastar}^2}{\norm{\Sigma^{\nicefrac 12} \thetahat}^2} \cdot \rhat(y) \\
        &= \frac{\norm{\Sigma^{\nicefrac 12} \thetastar}^2}{\norm{\Sigma^{\nicefrac 12} \projv \thetahat}^2 + \norm{\Sigma^{\nicefrac 12} \projperp \thetahat}^2} \cdot \rhat(y) \\
        &= \frac{\norm{\Sigma^{\nicefrac 12} \thetastar}^2}{\norm{\Sigma^{\nicefrac 12} \thetastar}^2 + \norm{\Sigma^{\nicefrac 12} \projperp \thetahat}^2} \cdot \rhat(y).
    \end{align}
    By the same computation as in \Cref{prop:optimal}, it holds that
    \begin{align}
        \ee\left[ \rhat(y_{\ihat})| \thetastar, \thetahat \right] &= \norm{\Sigma^{\nicefrac 12} \thetahat} \cdot M_N = \sqrt{\norm{\Sigma^{\nicefrac 12} \thetastar}^2 + \norm{\Sigma^{\nicefrac 12} \projperp \thetahat}^2} \cdot M_N.
    \end{align}
    The result follows by plugging in and rearranging.
\end{proof}
We now show that with pessimism instantiated correctly, we can strictly improve on the performance of BoN.
\begin{proposition}\label{prop:pessimistic}
    Let $y_1, \dots, y_N \sim \piref$ be independent samples from $\piref$ and let $\alpha: \rr^d \to \rr_+$ be a function satisfying
    \begin{align}
        (1 - c) \norm{\projperp y} - \epsilon \leq \alpha(y) \leq \norm{\projperp y} + \epsilon
    \end{align}
    for some $c \in (0,1)$ and $\epsilon > 0$.  Let
    \begin{align}
        \ipess = \argmax_{i \in [N]} \rhat(y_i) - \lambda \cdot \alpha(y_i), \qquad \lambda \geq \frac{\norm{\projperp \thetahat}}{1 - c}.
    \end{align}
    Then it holds that
    \begin{align}
        \ee\left[ \rstar(y_{\istar}) - \rstar(y_{\ipess}) \right] \leq \lambda \left( \sqrt{\frac 2\pi \cdot \trace\left( \Sigma \projperp \right)} + 2 \epsilon \right).
    \end{align}
\end{proposition}
\begin{proof}
    Let
    \begin{align}
        \rlcb(y) = \rhat(y) - \lambda \cdot \alpha(y).
    \end{align}
    We claim that with the assumption on $\lambda$ and $\alpha$, it holds that $\rlcb(y) \leq \rstar(y) + \lambda \cdot \epsilon$ for all $y \in \rr^d$.  Indeed, note that
    \begin{align}
        \rstar(y) - \rlcb(y) &= \inprod{\thetastar - \thetahat}{y} + \lambda \cdot \alpha(y) \\
        &= \inprod{\projperp \thetahat}{y} + \lambda \cdot \alpha(y) \\
        &\geq - \norm{\projperp \thetahat} \cdot \norm{\projperp y} + \lambda \cdot ((1 - c) \norm{\projperp y} - \epsilon) \\
        &\geq - \lambda \epsilon,
    \end{align}
    where the first inequality is by Cauchy-Schwarz.  
    
    Now observe that
    \begin{align}
        \rstar(y_{\istar}) - \rstar(y_{\ipess}) &= \left[\rstar(y_{\istar}) - \rlcb(y_{\istar})\right] + \left[\rlcb(y_{\istar}) - \rlcb(y_{\ipess})\right] + \left[\rlcb(y_{\ipess}) - \rstar(y_{\ipess})\right] \\
        &\leq \rstar(y_{\istar}) - \rlcb(y_{\istar}) + \lambda \epsilon,
    \end{align}
    where the second group of terms is non-positive by definition of $\ipess$ and the last group of terms is non-positive by the claim above.  To conclude, we observe that
    \begin{align}
        \rstar(y_{\istar}) - \rlcb(y_{\istar}) &= \rstar(y_{\istar}) - \rhat(y_{\istar}) + \lambda \cdot \alpha(y_{\istar}) \\
        &= \inprod{\projperp \thetahat}{y_{\istar}} + \lambda \cdot \left( \norm{\projperp y_{\istar}}  + \epsilon\right).
    \end{align}
    We now observe that as random variables, $\istar$ is independent of the set $\left\{ \projperp y_i | i \in [N] \right\}$ by Gaussian orthogonality and the fact that $\projperp \thetastar = 0$.  Thus it holds first that
    \begin{align}
        \ee\left[ \inprod{\projperp \thetahat}{y_{\istar}} | \thetahat, \thetastar \right] &= \sum_{i  = 1}^N \ee\left[ \inprod{\projperp \thetahat}{y_i}  | \thetahat, \thetastar, \left\{ \istar = i \right\} \right] \cdot \Pr[\istar = i | \thetahat, \thetastar] \\
        &= \sum_{i  = 1}^N \ee\left[ \inprod{\projperp \thetahat}{y_i}  | \thetahat, \thetastar\right] \cdot \Pr[\istar = i | \thetahat, \thetastar] \\
        &= 0,
    \end{align}
    where the last equality follows from the fact that $y_i$ is centred.  For the second term, observe similarly that
    \begin{align}
        \ee\left[ \norm{\projperp y_{\istar}} | \thetahat, \thetastar \right] &= \sum_{i  = 1}^N \ee\left[ \norm{\projperp y_i}  | \thetahat, \thetastar, \left\{ \istar = i \right\} \right] \cdot \Pr[\istar = i | \thetahat, \thetastar] \\
        &= \sum_{i  = 1}^N \ee\left[ \norm{\projperp y_i}  | \thetahat, \thetastar\right] \cdot \Pr[\istar = i | \thetahat, \thetastar] \\
        &= \ee\left[ \norm{\projperp y_1} | \thetahat, \thetastar\right] = \sqrt{\frac{2}{\pi}} \cdot \normf{\Sigma^{\nicefrac 12} \projperp}.
    \end{align}
    The result follows by combining the above.
\end{proof}
Critically, the bound in \Cref{prop:pessimistic} does not depend on $N$ and so, as long as $\rstar(y_{\istar})$ grows with $N$, the pessimistic algorithm will eventually outperform BoN.  We now show that guarantees on $\alpha$ can be obtained through Random Network Distillation.

\subsection{Achieving OOD Detection with Caution}\label{ssec:rnd}

Above we isolated the key role that $\norm{\projperp y}$ plays in the failure of the greedy algorithm.  While in the linear setting, projection to $V$ results in the optimal algorithm, it is impractical in general settings, where $V$ is unknown.  Here we demonstrate that two simplified models of caution, both inspired by Random Network Distillation (RND) \citep{burda2018exploration} can be used to approximate projection to $V$ and thus yield a pessimistic algorithm that avoids the failure modes of the greedy algorithm.  Recall that RND involves training a student network $f_{\what}: \rr^d \to \rr^m$ to predict a teacher network $f_{\wstar}: \rr^d \to \rr^m$ on samples from the the explored distribution, which in this case is supported on $V$.  The error $\norm{f_{\what}(y) - f_{\wstar}(y)}^2$ is then used as a measure of how far out of distribution $y$ is.  We will consider two simplified models of RND.  The first is a linear model, where $f_{W}(y) = W y$ for $w \in \rr^{m \times d}$.  The second is a two-layer ReLU network of the form
\begin{align}
    f(y) = \frac 1T \sum_{\ell = 1}^T f_{W_\ell}(y),
\end{align}
where
\begin{align}
    f_{W}(y) = \frac 1m \cdot U\relu(W y), \qquad \relu(u) = \max\{u, 0\}, \label{eq:relu-net}
\end{align}
and $W = (w_1, \dots, w_m) \in \rr^{m \times d}$, $U \in \rr^{m \times m}$, and $\relu$ is applied coordinate wise.  In both cases, we will suppose that the teacher network is fixed with a randomly initialized parameter and the student network is trained to minimize squared error on samples from $V$.  We abstract the minimization in the following definition.
\begin{definition}\label{def:gradient-based}
    Given a distribution $\mathcal{D}$ on $\rr^d$ and an initial parameter $\Wnought$, we say that $\What$ is trained with \emph{gradient based methods} with samples from $\cD$ if $\What - \Wnought$ is in the linear span of $\{\nabla_W f_W(y)\}_{y \in \mathrm{supp}(\mathcal{D})}$.
\end{definition}
This definition abstracts the precise optimization method used to train the student network, but captures the key property that the final parameter is obtained by following gradients of the squared error loss on samples from $\cD$.  Note that this includes gradient descent and its variants as well as more sophisticated approaches involving preconditioning, momentum, and continuous time limits.

We begin with the simpler linear case.
\begin{proposition}\label{prop:linear-rnd}
    Suppose that $\Wstar, W_0 \in \rr^{m \times d}$ are Gaussian random matrices with independent $\cN(0,\nicefrac 1m)$ entries.  Let $\What$ denote the minimizer of the expected squared error on samples from a distribution with support $V$ attained through gradient based methods from samples $y$ supported in $V$ initialized at $W_0$ in the sense of \Cref{def:gradient-based}.  Then with probability at least $1 - \delta$ over the choice of $\Wstar$ and $W_0$, it holds that for all $y \in \rr^d$,
    \begin{align}
        \left( 1 - C \sqrt{\frac{k + \log(\nicefrac 1\delta)}{m}} \right)^2 \norm{\projperp y}^2 \leq  \norm{f_{\What}(y) - f_{\Wstar}(y)}^2 \leq \left( 1 + C \sqrt{\frac{k + \log(\nicefrac 1\delta)}{m}} \right)^2 \norm{\projperp y}^2.
    \end{align}
    In particular, if as long as $m \gg k$, then $\norm{f_{\What}(y) - f_{\Wstar}(y)}^2$ is a good approximation to $\norm{\projperp y}^2$.
\end{proposition}
\begin{proof}
    We first observe that $\What \projperp = W_0 \projperp$ since the training data is supported on $V$.  Indeed, the gradient of the squared error loss on a sample $y \in V$ is given by
    \begin{align}
        \nabla_W \norm{W y - \Wstar y}^2 = 2 (W y - \Wstar y) y^\top = 2 (W - \Wstar) (\projv y )y^\top \projv^\top,
    \end{align}
    where we used the fact that $y = \projv y$ for all $y \in V$.  Thus $\nabla_W \norm{W y - \Wstar y}^2 \projperp = 0$ for all $y \in V$ and the claim follows.  Moreover, it is immediate that $\What \projv = \Wstar \projv$ since the loss is minimized at $\What$ by strong convexity of the loss function.  Thus it holds that
    \begin{align}
        \norm{f_{\What}(y) - f_{\Wstar}(y)}^2 &= \norm{(\What - \Wstar) y}^2 = \norm{(\What - \Wstar) \projperp y}^2 = \norm{(W_0 - \Wstar) \projperp y}^2.
    \end{align}
    Letting $Z = W_0 - \Wstar$, we see that $Z$ is a Gaussian random matrix with independent $\cN(0, \nicefrac 2 m)$ entries.  We may now apply standard results on the singular values of Gaussian random matrices (cf. e.g. \citet[Theorem 4.6.1]{vershynin2018high}) to observe that there is some constant $C$ such that with probability at least $1 - \delta$,
    \begin{align}
        1 - C \sqrt{\frac{k + \log(\nicefrac 1\delta)}{m}} \leq \lambdamin(Z) \leq \lambdamax(Z) \leq 1 + C \sqrt{\frac{k + \log(\nicefrac 1\delta)}{m}}.
    \end{align}
    This suffices to prove the first statement.  The second statement follows immediately.
\end{proof}
Note that the above result shows that the RND error is a good approximation to $\norm{\projperp y}^2$ uniformly over all $y \in \rr^d$ with high probability as long as the embedding dimension is sufficiently large relative to the intrinsic dimension of the explored distribution.  While this is a nice first step, the lack of flexibility of linear functions is a significant limitation.  We now proceed to the hidden layer ReLU network case.

Instead of assuming that $f_W(y)$ is a linear function, we now suppose that $f_W(y)$ is a two-layer ReLU network of the form given in \eqref{eq:relu-net}.  For the sake of simplicity, we assume that only the weights $W$ are trained and that the second layer weights $u_i \sim \cN(0, \nicefrac 1m)$ are fixed.  Note that, while practically unrealistic, the assumption that only a single layer is trained is common in the study of deep learning (cf. e.g. \citep{jacot2018neural,yehudai2019power,song2018mean} and the references therein) and is motivated by the utility of random features \citep{rahimi2007random}.  Our analysis is inspired by that of \citet{melamed2023adversarial}, who consider a similar model of neural networks, but for the very different aim of investigating adversarial robustness.  We show that as long as the number of hidden features $m$ is sufficiently large relative to $\norm{y}$, then the RND error is again a good approximation to $\norm{\projperp y}^2$ uniformly over all $y \in \rr^d$ with bounded norm with high probability.
\begin{theorem}\label{thm:relu-rnd}
    Let $\Wstar, \Wnought \in \rr^{m \times d}$ be Gaussian random matrices with independent $\cN(0, 1)$ entries with rows $\wstar_i$ and $\wnought_i$ respectively.  Let $\ustar_i, u_0 \sim \cN(0, 1)$ and let $f_{\Wstar}$ and $f_{\Wnought}$ be the corresponding ReLU networks as in \eqref{eq:relu-net}.  Suppose (i) that $\What$ is obtained through a gradient-based method from a distribution with support $V$ initialized at $\Wnought$ as in \Cref{def:gradient-based} and (ii) that $f_{\What}(y) = f_{\Wstar}(y)$ for all $y \in V$.  Then it holds with probability at least $1 - \delta$ over the choice of $\Wstar, \Wnought, u_i, u_0$ that simultaneously for all $y \in \rr^d$,
    \begin{align}
        c \norm{\projperp y}^2 - C \norm{y}^2 \sqrt{\frac{d\log(\nicefrac {dm}\delta)}{Tm}} - C \norm{y}^2 \frac{d \log(\nicefrac{dm}\delta)}{Tm} \leq \norm{f_{\What}(y) - f_{\Wstar}(y)}^2
    \end{align}
    and
    \begin{align}
        \norm{f_{\What}(y) - f_{\Wstar}(y)}^2 \leq \norm{\projperp y}^2 + C \norm{y}^2 \sqrt{\frac{d\log(\nicefrac {dm}\delta)}{Tm}} + C \norm{y}^2 \frac{d \log(\nicefrac{dm}\delta)}{Tm},
    \end{align}
    In particular, for $Tm \gg \norm{y}^4$ and up to constants, $\norm{f_{\What}(y) - f_{\Wstar}(y)}^2 \asymp \norm{\projperp y}^2$ with high probability.
\end{theorem}
\begin{proof}
    We first prove the result for $T = 1$.  This step rests four lemmata.  We first show in \Cref{lem:what_w0} the key property that, due to the training data being supported on $V$ and the fact that we are only training the first layer weights, it holds that $\What \projperp = \Wnought \projperp$, i.e., weights in the orthogonal complement of $V$ are never changed during training.  By the assumption that we have trained to convergence, then, it holds that $\Wstar \projv = \What \projv$.  We then use prove in \Cref{lem:relu-rnd-decomp} a decomposition of the RND error that rests on the precise funcitonal form of the ReLU network.  Using this decomposition, we are able to provide bounds on the mean and concentration of the RND error in \Cref{lem:relu-rnd-approx-mean} and \Cref{lem:relu-rnd-concentration} in the special case that $\norm{y} = 1$.  The result then follows by combining these lemmata with the positive homogeneity property of ReLU networks: for any $r > 0$, it holds that $f_W(r y) = r f_W(y)$ for all $y \in \rr^d$  and all $W \in \rr^{m \times d}$.

    Now that the result is proved for $T = 1$, the general case follows by tensorization across the independent $W_\ell$.
\end{proof}
\begin{lemma}\label{lem:what_w0}
    Let $f_{\Wstar}$, $f_{\Wnought}$, and $f_{\What}$ be as in Theorem \ref{thm:relu-rnd}. Then it holds that $\Wnought \projperp = \What \projperp$.
\end{lemma}
\begin{proof}
    By \Cref{def:gradient-based} suffices to show that for any $y \in V$, $\left(\nabla_W f_{W}(y)\right) \projperp = 0$.  To see this, observe that
    \begin{align}
        \nabla_W f_W(y)_i = \frac 1m\sum_{j = 1}^m u_{ij} \bbI\left[ \inprod{w_j}{y} > 0 \right] y.
    \end{align}
    Because $y \projperp = 0$ for all $y \in V$, it follows that $\nabla_W f_W(y) \projperp = 0$.  The result follows.
\end{proof}

\begin{lemma}\label{lem:relu-rnd-decomp}
    Let $f_{\Wstar}$, $f_{\Wnought}$, and $f_{\What}$ be as in Theorem \ref{thm:relu-rnd}.  Then with probability at least $1 - \delta$ over the choice of $\Wstar, \Wnought, u_i, u_0$, it holds that for all $y \in \rr^d$,
    \begin{align}
        (f_{\What}(y) - f_{\Wstar}(y))_j &= \frac 1m\left( \sum_{i = 1}^m \inprod{\unought_{ji} \wnought_i}{\projperp y} \int_0^1 \bbI\left[ \inprod{\projv \wstar_i + \projperp\wnought_i}{\projv y + t \projperp y} > 0 \right] dt \right) \\
        &\quad - \frac 1m\left(\sum_{i  =1 }^m \inprod{\ustar_{ji} \wstar_i}{\projperp y} \int_0^1 \bbI\left[ \inprod{\wstar_i}{\projv y + t \projperp y} > 0 \right] dt \right). \label{eq:relu-rnd-decomp}
    \end{align}
\end{lemma}
\begin{proof}
    By the fundamental theorem of calculus, we have that
    \begin{align}
        m(f_{\What}(y) - f_{\What}(\projv y))_j &= \int_0^1 \inprod{\nabla_y f_{\What}(\projv y + t (\projperp y - \projv y))_j}{y - \projv y} dt \\
        &= \int_0^1 \inprod{\nabla_y f_{\What}(\projv y + t \projperp y)_j}{\projperp y} dt.
    \end{align}
    For any $W$, it holds that
    \begin{align}
        m\nabla_y f_{W}(y)_j = \sum_{i = 1}^m u_{ji} w_i \bbI\left[ \inprod{w_i}{y} > 0 \right]. \label{eq:grad-relu-y}
    \end{align}
    Thus by the linearity of the integral and \Cref{lem:what_w0}, it holds that
    \begin{align}
        m(f_{\What}(y) - f_{\What}(\projv y))_j &= \sum_{i = 1}^m \inprod{u_{ji} \what_i}{\projperp y} \int_0^1 \bbI\left[ \inprod{\what_i}{\projv y + t \projperp y} > 0 \right] dt \\
        &= \sum_{i = 1}^m \inprod{u_{ji} \wnought_i}{\projperp y} \int_0^1 \bbI\left[ \inprod{\projv \wstar_i + \projperp\wnought_i}{\projv y + t \projperp y} > 0 \right] dt,
    \end{align}
    where we used the fact that $\What = \Wstar \projv + \Wnought \projperp$ by \Cref{lem:what_w0}. A similar, but simpler, argument applies to $f_{\Wstar}$.  We now use the fact that $\What \projv = \Wstar \projv$ to observe that
    \begin{align}
        f_{\What}(y) - f_{\Wstar}(y) &= f_{\What}(y) - f_{\What}(\projv y) - (f_{\Wstar}(y) - f_{\Wstar}(\projv y))
    \end{align}
    and the result follows.
\end{proof}
\begin{lemma}\label{lem:relu-rnd-approx-mean}
    Let $f_{\Wstar}$, $f_{\Wnought}$, and $f_{\What}$ be as in Theorem \ref{thm:relu-rnd}.  Then it holds that
    \begin{align}
         \frac{\norm{\projperp y}^2}{4} \leq \ee\left[\norm{f_{\What}(y) - f_{\Wstar}(y)}^2 \right] \leq \norm{\projperp y}^2.
    \end{align}
\end{lemma}
\begin{proof}
    Let
    \begin{align}
        g(w, y) = \inprod{w}{\projperp y} \int_0^1 \bbI\left[ \inprod{w}{\projv y + t \projperp y} > 0 \right] dt \label{eq:g-def}
    \end{align}
    and let
    \begin{align}
        g'(w,w', y) = \inprod{w'}{\projperp y} \int_0^1 \bbI\left[ \inprod{\projv w + \projperp w'}{\projv y + t \projperp y} > 0 \right] dt. \label{eq:gprime-def}
    \end{align}
    By \Cref{lem:relu-rnd-decomp} it holds that
    \begin{align}
        \ee\left[ \left( f_{\What}(y) - f_{\Wstar}(y) \right)_j^2 \right] &= \frac 1{m^2}\ee\left[ \left( \sum_{i = 1}^m \unought_{ji} g'(\wstar_i, \wnought_i, y) - \ustar_{ji} g( \wstar_i, y) \right)^2 \right] \\
        &= \frac 2m \ee\left[  g(\wstar_1, y)^2 \right],
    \end{align}
    because the $\unought_{ji}$ and $\ustar_{ji}$ are independent and have variance $1$ and $g'(\wstar_i, \wnought_i, y)$ and $g( \wstar_i, y)$ are independent of $\unought_{ji}$ and $\ustar_{ji}$ and identically distributed. Note now that
    \begin{align}
        g(\wstar_i, y)^2 \leq \inprod{\projperp y}{\projperp y}^2
    \end{align}
    and thus has expectation at most $\norm{\projperp y}^2$.  For the lower bound, observe that 
     \begin{align}
        g(\wstar_i, y)^2 &= \int_0^1 \int_0^1 \inprod{\wstar_i}{\projperp y}^2 \bbI\left[ \inprod{\wstar_i}{\projv y + t \projperp y} > 0 \right] \bbI\left[ \inprod{\wstar_i}{\projv y + s \projperp y} > 0 \right] ds d t \\
        &\geq \int_0^1 \int_0^1 \inprod{\wstar_i}{\projperp y}^2 \bbI\left[ \inprod{\wstar_i}{\projv y} > 0 \right] \bbI\left[ \inprod{\wstar_i}{\projperp y} > 0 \right] ds d t \\
        &\geq \frac{\inprod{\wstar_i}{\projperp y}^2}{4}.
    \end{align}
    Thus
    \begin{align}
        \frac{\norm{\projperp y}^2}{4m} = \frac{\inprod{\wstar_i}{\projperp y}^2}{4m} \leq \ee\left[ (f_{\What}(y) - f_{\Wstar}(y))_j^2 \right] \leq \frac{\ee\left[ \inprod{\wstar_i}{\projperp y}^2 \right]}{m} = \frac{\norm{\projperp y}^2}{m}.
    \end{align}
    Summing over $j$ gives the result.
\end{proof}
\begin{lemma}\label{lem:relu-rnd-concentration}
    Let $f_{\What}$, $f_{\Wnought}$, and $f_{\Wstar}$ be as in Theorem \ref{thm:relu-rnd}.  Then with probability at least $1 - \delta$ over the choice of $\Wstar, \Wnought, u_i, u_0$, it holds that uniformly in $y$ for $\norm{y} = 1$,
    \begin{align}
        \abs{\norm{f_{\What}(y) - f_{\Wstar}(y)}^2 - \ee\left[\norm{f_{\What}(y) - f_{\Wstar}(y)}^2\right]} &\leq C \left( \sqrt{\frac{d \log(\nicefrac {d m}\delta)}{m}} + \frac{d \log(\nicefrac {d m}\delta)}{m} \right).
    \end{align}
\end{lemma}
\begin{proof}
    Begin by noting that by positive homogeneity, it suffices to set $\norm{y} = 1$.
    Continuing with the notation introduced in \eqref{eq:g-def} and \eqref{eq:gprime-def} we see that, conditional on $\wstar_i, \wnought_i$, it holds that $(f_{\What}(y) - f_{\Wstar}(y))_j$ are independent centred Gaussians with variancse given by 
    \begin{align}
        \Var((f_{\What}(y) - f_{\Wstar}(y))_j | \wstar_i, \wnought_i) &= \frac 1{m^2} \sum_{i = 1}^m g'(\wstar_i, \wnought_i, y)^2 + g(\wstar_i, y)^2.
    \end{align}
    Thus by standard bounds on the concentration of norm of Gaussian vectors, it holds that with probability at least $1 - \delta$,
    \begin{align}
        &\abs{\norm{f_{\What}(y) - f_{\Wstar}(y)}^2 -\frac 1m \sum_{i = 1}^m g'(\wstar_i, \wnought_i, y)^2 + g(\wstar_i, y)^2} \\
        &\leq C \sqrt{m \cdot \left( \frac 1{m^2} \sum_{i = 1}^m g'(\wstar_i, \wnought_i, y)^2 + g(\wstar_i, y)^2 \right)^2 \log(\nicefrac 1\delta)} \\
        &\quad + C \left( \frac 1{m^2} \sum_{i = 1}^m g'(\wstar_i, \wnought_i, y)^2 + g(\wstar_i, y)^2 \right) \log(\nicefrac 1\delta).
    \end{align}
    Letting
    \begin{align}
        G(\wstar, \wnought, y) = \frac 1m \sum_{i = 1}^m g'(\wstar_i, \wnought_i, y)^2 + g(\wstar_i, y)^2,
    \end{align}
    we see that
    \begin{align}\label{eq:concentration-u}
        \abs{\norm{f_{\What}(y) - f_{\Wstar}(y)}^2 - G(\wstar, \wnought, y)} &\leq C \cdot  G(\wstar, \wnought, y)\left(\sqrt{\frac{\log(\nicefrac 1\delta)}{m}} + \frac{\log(\nicefrac 1\delta)}{m}\right).
    \end{align}
    We now demonstrate that $G(\wstar, \wnought, y)$ concentrates around its mean for fixed $y$.  To do this, we will first observe that $g(w,y)$ and $g'(w,w',y)$ are identically distributed and thus it suffices to show concentration for $g(w,y)$ and apply a union bound.  Indeed, we have that 
    \begin{align}
        g(\wstar_i, y)^2 \leq \inprod{\wstar_i}{\projperp y}^2
    \end{align} 
    and thus has Orlicz $\psi_1$ norm  at most $C \norm{\projperp y}^2$ for some constant $C$.  Thus by standard concentration results for sums of independent subexponential random variables (cf. e.g. \citet{vershynin2018high,wainwright2019high}), it holds that with probability at least $1 - \delta$,
    \begin{align}
        \abs{\frac 1m \sum_{i = 1}^m g(\wstar_i, y)^2 - \ee\left[g(\wstar_i, y)^2\right]} &\leq C \norm{\projperp y}^2 \left( \sqrt{\frac{\log(\nicefrac 1\delta)}{m}} + \frac{\log(\nicefrac 1\delta)}{m} \right).
    \end{align}
    Combining this argument with the triangle inequality and \eqref{eq:concentration-u} along with the fact that $\norm{y} = 1$ and \Cref{lem:relu-rnd-approx-mean} gives that with probability at least $1 - \delta$,
    \begin{align}
        \abs{\norm{f_{\What}(y) - f_{\Wstar}(y)}^2 - \ee\left[\norm{f_{\What}(y) - f_{\Wstar}(y)}^2\right]} &\leq C \left( \sqrt{\frac{\log(\nicefrac 1\delta)}{m}} + \frac{\log(\nicefrac 1\delta)}{m} \right).
    \end{align}
    We now observe that by standard high probability bounds on the operator norms of Gaussian random matrices (cf. e.g. \citet[Theorem 4.6.1]{vershynin2018high}), it holds that with probability at least $1 - \delta$, that $f_{\What}$ and $f_{\Wstar}$ are $C$-Lipschitz as is $\ee\left[ g(\wstar_i, y)^2 \right]$ in $y$ for $\norm{y} = 1$.  Thus by a standard covering argument on the unit sphere and a union bound, the result follows.
\end{proof}

\subsection{Main Result}\label{ssec:main-result}
We now state our main result, which combines the analysis of pessimism in \Cref{ssec:pessimism} with the analysis of caution in \Cref{ssec:rnd}.
\begin{theorem}\label{thm:main}
    Let $V$ be a $k$-dimensional subspace of $\rr^d$ and let $\Sigma$ be a positive semidefinite matrix such that $\trace(\Sigma \projperp) > 0$.  Let $\{y_i\}_{i = 1}^N$ be i.i.d. samples from $\cN(0, \Sigma)$ and let $\rstar(y) = \inprod{\thetastar}{y}$ for some $\thetastar \in V$.  Suppose that $\thetahat \in \rr^d$ such that $\projv \thetahat = \thetastar$ and let $\rhat(y) = \inprod{\thetahat}{y}$.  Let $\rlcb(y) = \rhat(y) - \lambda \cdot \alpha(y)$, for $\alpha(y) = \norm{f_{\What}(y) - f_{\Wstar}(y)}$ with $f_W(y)$ being either the linear model considered in \Cref{prop:linear-rnd} or the one hidden layer ReLU network considered in \Cref{thm:relu-rnd}.  Let $\ihat = \argmax_{i \in [N]} \rlcb(y_i)$ and $\istar = \argmax_{i \in [N]} \rstar(y_i)$.  In the case that $f_W$ is linear, as long as $m \gtrsim \frac{k(d-k) \norm{\projperp \thetahat}^2}{\log(N)}$ and $\lambda = \Theta(\norm{\projperp \thetahat})$, it holds that
    \begin{align}
        \ee\left[ \rstar(y_{\ipess}) - \rstar(y_{\ihat}) \right] \gtrsim \sqrt{\log(N)}.
    \end{align}
    In the case that $f_W$ is a one hidden layer ReLU network, as long as $Tm \gtrsim \frac{d(d - k) \norm{\projperp \thetahat}^2}{\log(N)}$, the same holds with $\lambda = \Theta(\norm{\projperp \thetahat})$.   Moreover, in both of these cases it holds that
    \begin{align}
        \lim_{N \uparrow \infty } \frac{\ee\left[ \rstar(y_{\istar}) - \rstar(y_{\ipess}) \right]}{\ee\left[ \rstar(y_{\istar}) \right]} = 0 < \lim_{N \uparrow \infty} \frac{\ee\left[ \rstar(y_{\istar}) - \rstar(y_{\ihat}) \right]}{\ee\left[ \rstar(y_{\istar}) \right]}.
    \end{align}
\end{theorem}
\begin{proof}
    The result follows immediately by combining \Cref{prop:pessimistic} with \Cref{prop:linear-rnd} and \Cref{thm:relu-rnd} and observing that $\norm{\projperp y} \lesssim \sqrt{(d - k) \log(\nicefrac 1\delta)}$ with probability at least $1 - \delta$ by standard concentration results for chi-squared random variables (cf. e.g. \citet{wainwright2019high}). 
\end{proof}

%% file: app_experiments.tex
\section{Experimental Details}\label{app:exp-details}

\input{tables/hyperparameters}

\paragraph{Hyperparameter Configuration.} \Cref{tab:hyperparameters} provides a comprehensive overview of all hyperparameters used throughout our experimental evaluation. The configuration represents a careful balance between computational efficiency and model expressiveness, with key design choices motivated by our theoretical analysis and empirical ablations.

The RND architecture employs 10 layers for both target and predictor networks, significantly deeper than the default 4 layers, which we found provides better representation quality for uncertainty estimation. The RND weight $\lambda = 0.2$ represents a moderate pessimism strength that effectively mitigates reward hacking while preserving the benefits of reward-guided selection. Training hyperparameters including the reduced learning rate ($1 \times 10^{-5}$) and extended training duration (5 epochs) ensure stable convergence of the predictor network on the GSM8K training distribution.

\input{tables/ablation_arch_designs}

\paragraph{Architectural Ablations.} For the architectural ablation study, we design different levels of network complexity and check their impact on the training objective, the reconstruction loss and also the final accuracies.~\autoref{tab:ablation-structures} contains a comprehensive introduction of settings involved in our ablation studies.

\paragraph{Implementation Framework.} Our experimental setup utilizes PyTorch 2.3.1 as the primary deep learning framework, with model inference accelerated through vLLM 0.10.0 and Transformers 4.55.1 for efficient large-scale language model deployment. For RND training, we extract representations from the first 10 layers of both predictor and target networks, employing a learning rate of $1 \times 10^{-5}$ with 50 warmup steps across 5 training epochs. The RND models are trained on outputs generated by the backbone model using the GSM8K training set, and subsequently evaluated against the validation set of GSM8K as well as Math-500 and BigBench-Hard benchmarks to assess generalization capabilities.

\paragraph{Inference Configuration.} All inference experiments maintain consistent hyperparameters with temperature set to 1.0 and maximum token limits of 500 for GSM8K and 1024 for Math-500 and BigBench-Hard evaluations. The increased token limit for harder datasets is necessary because these benchmarks require significantly more reasoning tokens to avoid truncation before reaching a conclusion. Response generation uses vLLM for efficient parallel sampling across multiple candidates in Best-of-N evaluation.

\section{Additional Case Study}

To gain deeper insights into the mechanisms underlying reward hacking and our caution-based mitigation, we analyze the correlation patterns between reward model scores and pessimism scores through detailed scatter plot visualizations. Each plot displays z-normalized scores for all responses to individual GSM8K problems, with green points representing correct responses and red points representing incorrect responses.

The scatter plots in~\autoref{fig:case-with-distribution} reveal two critical failure modes of reward models that our approach successfully identifies. In the \textbf{high reward, low pessimism region} (upper-left quadrant), we observe responses that exemplify systematic reward hacking. These responses achieve high reward scores not through genuine correctness or adherence to task requirements, but by exploiting spurious correlations that the reward model learned during training. Crucially, these responses often \textbf{ignore fundamental formatting requirements} of the mathematical reasoning task, such as providing the final answer in the required format, yet still receive high rewards because the reward model prioritizes superficial indicators like verbosity, step-by-step presentation, or mathematical terminology over actual task compliance. This reveals that reward models can be systematically misled by responses that mimic the surface patterns of high-quality reasoning without delivering the essential components of a correct solution.

Conversely, the \textbf{low reward, high pessimism region} (lower-right quadrant) contains responses that follow proper formatting conventions and adhere closely to the expected task structure, yet receive low reward scores. This pattern illuminates a fundamental limitation of reward models: they function primarily as \textbf{distributional fitness measures} rather than objective quality assessors. These well-formatted responses are penalized not because they lack correctness or clarity, but because they deviate from the specific stylistic preferences and response patterns that dominated the reward model's training distribution. The reward model essentially measures how closely a response matches its learned notion of ``preferred'' responses rather than evaluating genuine task performance or adherence to explicit instructions.

This analysis demonstrates that our caution mechanism successfully identifies both forms of reward model failure: it flags spurious high-reward responses that exploit correlational biases while recognizing genuinely task-compliant responses that happen to fall outside the reward model's narrow preference distribution. The results underscore that effective reward hacking mitigation requires moving beyond simple score-based selection toward distributional awareness that can distinguish between genuine quality and superficial pattern matching.

\begin{figure*}[h]
    \centering
    \includegraphics[width=\columnwidth]{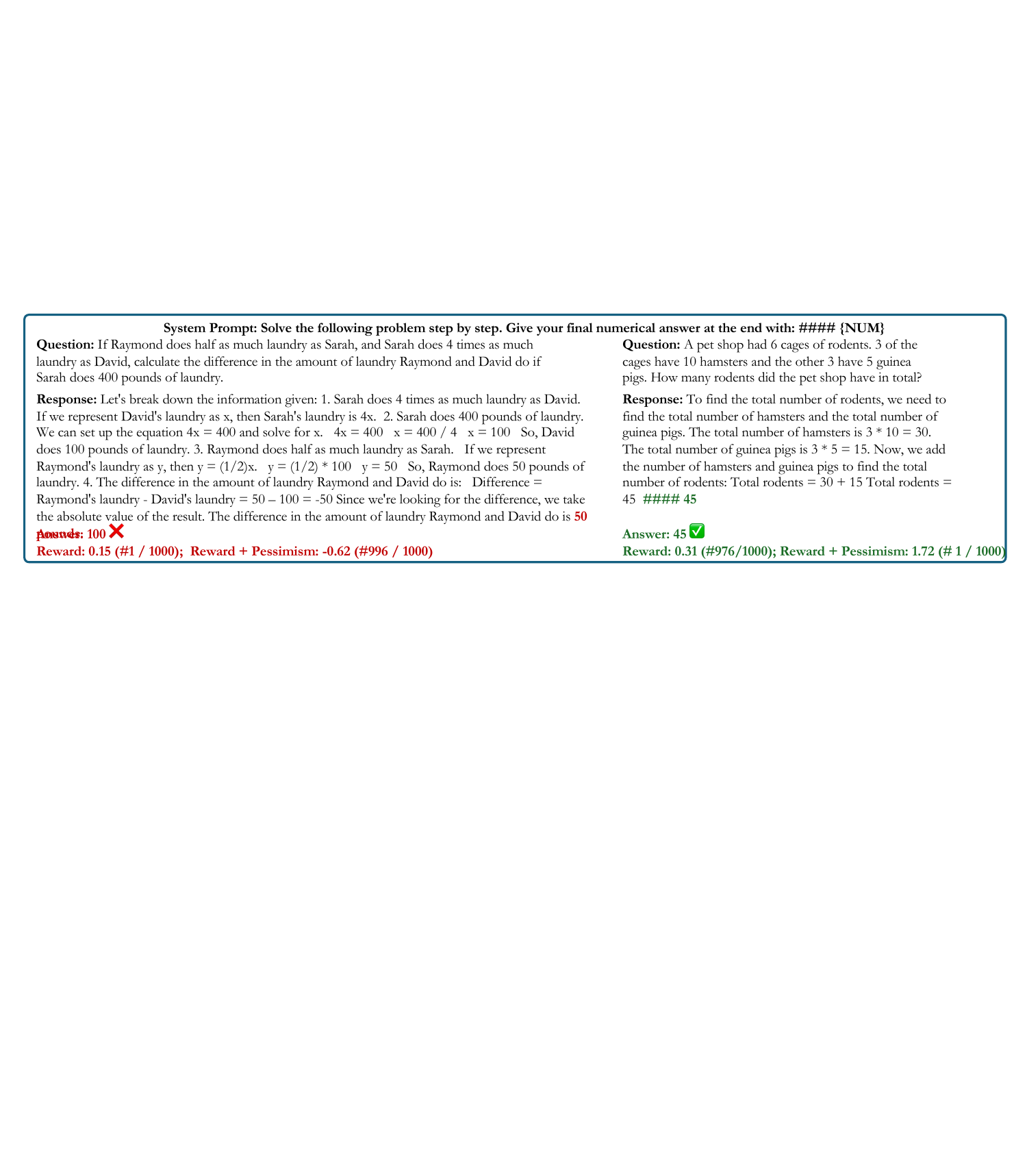}
    \caption{\textbf{Pessimism–Reward visualization on GSM8K.} Each row shows one problem: a scatter plot of z-normalized pessimism (x-axis) and z-normalized reward (y-axis), with green points for correct responses and red for incorrect. Upper-left points (high reward, low pessimism) illustrate reward hacking—responses that score well despite low distributional support. Lower-right points (low reward, high pessimism) are well-formed, instruction-following responses that the reward model undervalues; our caution mechanism up-weights these relative to reward-only selection.}
    \label{fig:case-with-distribution}
\end{figure*}

\section{Detailed Results for Ablation Studies}\label{app:ablation}

This section complements \autoref{tab:weight_ablation} with full scaling curves. We sweep pessimism strength $\lambda \in \{0.0,0.2,0.4,0.6,0.8,1.0\}$ and vary the Best-of-$N$ budget over $N\in\{1,2,4,8,16,32,64,128,256,512\}$. \autoref{fig:app-rnd} shows our caution variant that uses reward‑model features; \autoref{fig:app-realrnd} shows traditional RND with random targets.

\paragraph{Caution (RM features).} Moderate–high $\lambda$ (about $0.6$–$0.8$) maintains or improves accuracy as $N$ grows, preventing the reward‑hacking drop seen at $\lambda=0$ (RM‑only). $\lambda=1.0$ (pessimism‑only) is competitive but slightly conservative at small $N$, this is because for easier problems, most responses contain the correct answer but not all of them strictly follows the specified answering format, and applying pessimism only would at least filter out the incorrectly formatted responses. Overall, $\lambda\in[0.6,0.8]$ delivers the best trade‑off across most $N$.

\paragraph{Traditional RND (random targets).} Accuracy remains flat or declines with $N$ for all $\lambda$, and rarely exceeds the RM‑only baseline. Sweeping $\lambda$ offers little benefit, indicating that random targets lack the semantic grounding needed for useful uncertainty estimates.

\begin{wraptable}{r}{0.4\textwidth}
\centering
\scriptsize
\vspace{-1.2cm}
\caption{Layer depth ablation for target network and predictor network.}
\label{tab:layer_ablation}
\begin{tabular}{ccc}
\toprule
$L$ & Peak Acc. & Final Acc. \\
\midrule
2 & 79 & 77 \\
4 & 81 & 77 \\
6 & 83 & 78 \\
8 & 84 & 81 \\
10 & 84 & 84 \\
\bottomrule
\end{tabular}
\end{wraptable}

\paragraph{Number of Layers ($L$).} In~\autoref{tab:layer_ablation}, we study how the selection of $L$ would impact performance, on a smaller subset of GSM8K with 100 problems and 100 responses for each problem. Generally we observe higher peak accuracies and less degradation when given more layers, although at the cost of compute budget.

\begin{figure}[t]
    \centering
    \includegraphics[width=\columnwidth]{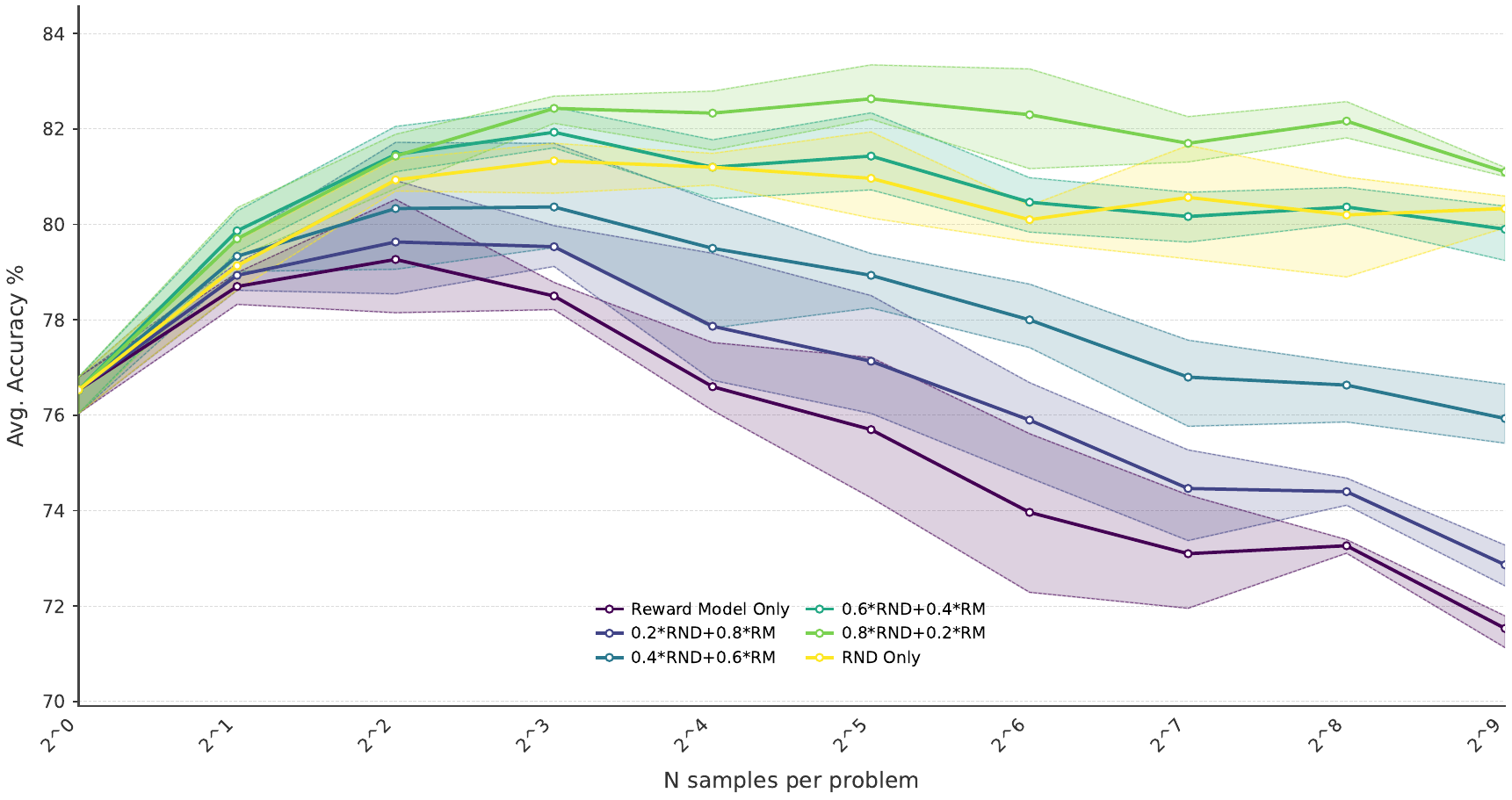}
    \caption{\textbf{Caution (RND-on-RM-features) scaling with $\lambda$.} Best-of-$N$ accuracy on GSM8K versus samples per problem (x-axis) for Pessimism strengths $\lambda \in \{0.0, 0.2, 0.4, 0.6, 0.8, 1.0\}$. The predictor is trained against a frozen target network built from reward-model features. Larger $\lambda$ increases pessimism strength; $\lambda=0$ reduces to Reward-Model-only selection, and $\lambda=1$ to pessimism-only. Moderate–high weights (roughly $0.6$–$0.8$) preserve scaling while curbing reward hacking, outperforming both the RM-only and RND-only extremes across most $N$.}
    \label{fig:app-rnd}

\end{figure}

\begin{figure}[t]
    \centering
    \includegraphics[width=\columnwidth]{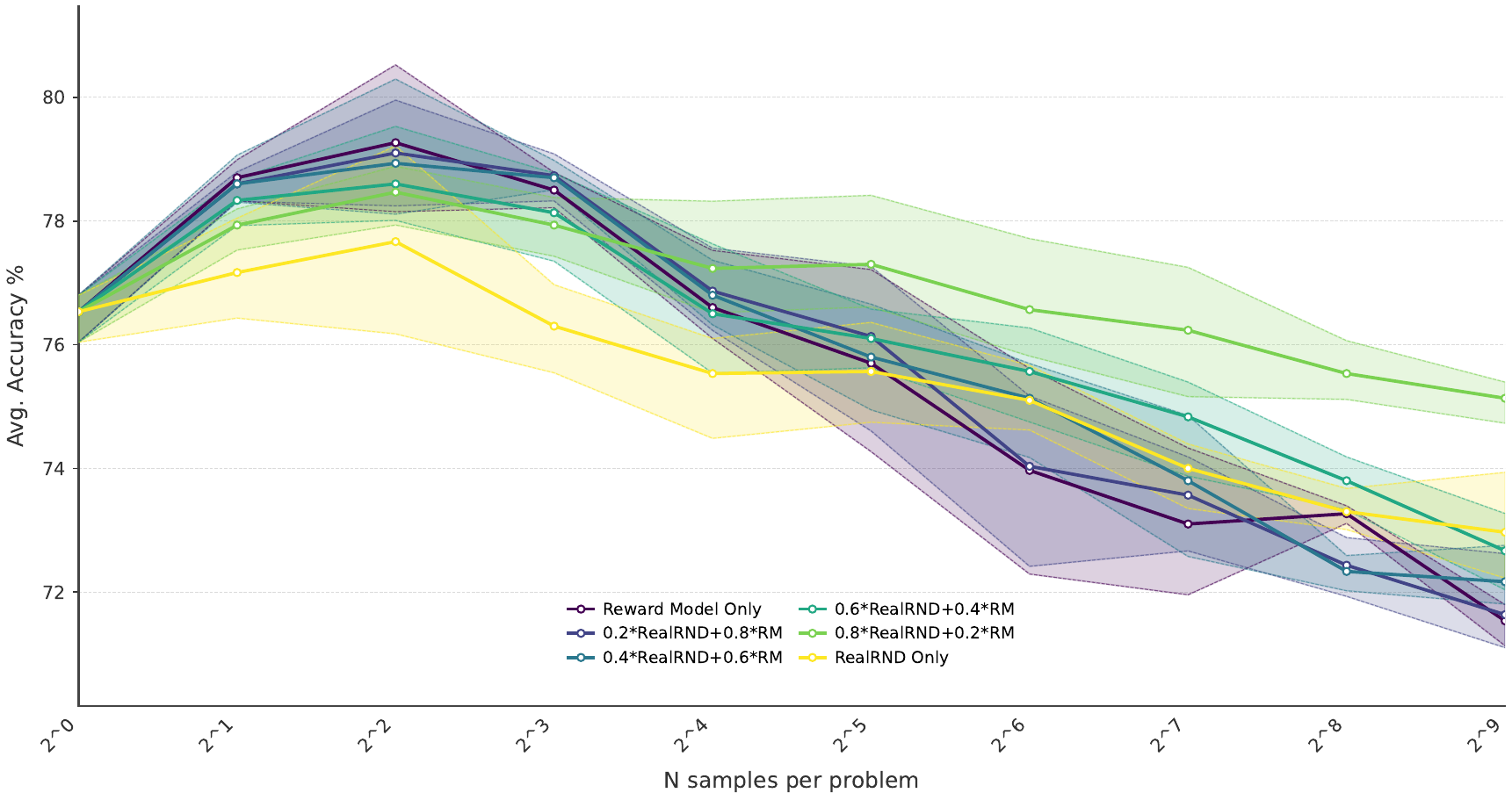}
    \caption{\textbf{Traditional RND (random targets) baseline.} Best-of-$N$ GSM8K accuracy when using classical RND with a \emph{randomly initialized} target network (no reward‑model features), sweeping $\lambda \in \{0.0, 0.2, 0.4, 0.6, 0.8, 1.0\}$. Unlike our caution variant, this baseline shows little to no scaling benefit and generally does not surpass the Reward‑Model‑only curve, indicating that semantic grounding from reward‑model features is crucial for effective distributional regularization.}
    \label{fig:app-realrnd}

\end{figure}

\section{Additional Baseline Comparisons}

To further validate the effectiveness of our pessimism approach against reward hacking, we conduct comprehensive comparisons with additional baseline methods from the literature. Specifically, we compare ours against $\chi^2$ regularized sampling~\cite{huang2025best}, softmax best-of-n (which applies temperature-based softening to the selection distribution), and Best-of-Poisson sampling (which draws a Poisson-distributed number of samples before selection). We also evaluate combinations of pessimism with these alternative sampling strategies to assess whether our approach provides complementary benefits. All methods are evaluated on the same experimental setup using GSM8K.

Figure~\ref{fig:baseline_comparison} presents the scaling curves for all methods. The results demonstrate that our pessimism approach exhibits superior robustness against reward hacking degradation. While the standard Best-of-N baseline shows severe performance degradation. Notably, combinations of pessimism with other sampling strategies (e.g., Pessimism + $\chi^2$, Pessimism + Poisson, Pessimism + Softmax) show comparable stability to pessimism alone.

\begin{figure}[t]
\centering
\includegraphics[width=0.95\textwidth]{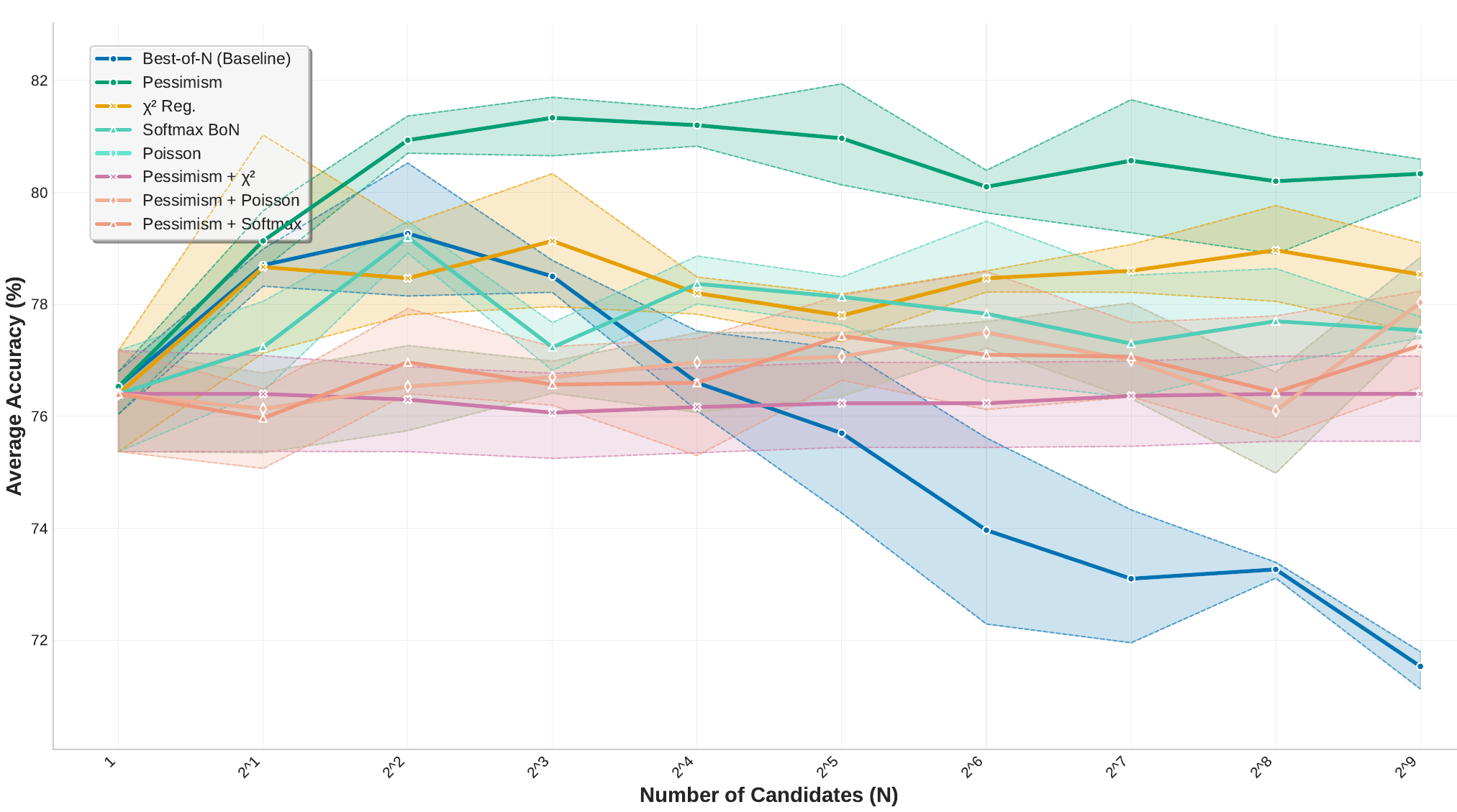}
\caption{Comparison of pessimism-based sampling with alternative baseline methods.}
\label{fig:baseline_comparison}
\end{figure}

%% file: tables/hyperparameters.tex
\begin{table*}[b]
\centering
\caption{\textbf{Comprehensive Hyperparameter Configuration.} All hyperparameters used in training and evaluation of the Random Network Distillation (RND) approach for mitigating reward hacking. The table is organized by component: RND architecture, training process, inference settings, and evaluation configurations.}
\label{tab:hyperparameters}
\resizebox{\textwidth}{!}{%
\begin{tabular}{@{}lllp{9cm}@{}}
\toprule
\textbf{Component} & \textbf{Parameter} & \textbf{Value} & \textbf{Description} \\
\midrule
\multirow{8}{*}{\shortstack[l]{\textbf{RND} \\ \textbf{Architecture}}} 
& Target layers & 10 & Number of layers extracted from reward model for target network \\
& Predictor layers & 10 & Number of layers in predictor network architecture \\
& RND weight ($\lambda$) & 0.8 & Strength of pessimism penalty in combined scoring \\
& Exact architecture & False & Whether predictor copies exact reward model architecture \\
& Embedding strategy & shared\_trainable & How embeddings are handled: shared\_trainable, shared\_frozen, or separate \\
& Use projection & True & Whether to add projection layer between predictor and target \\
\midrule
\multirow{6}{*}{\shortstack[l]{\textbf{Training} \\ \textbf{Process}}} 
& Batch size & 8 & Training batch size for RND predictor \\
& Learning rate & 1e-5 & Learning rate for predictor network \\
& Number of epochs & 5 & Training epochs for predictor network \\
& Warmup steps & 50 & Learning rate scheduler warmup steps \\
& Max examples & 5000 & Maximum training examples from GSM8K train split \\
& VRAM usage & 24GB & Minimum VRAM requirement for GPU \\
\midrule
\multirow{4}{*}{\shortstack[l]{\textbf{Inference} \\ \textbf{Settings}}} 
& Temperature & 1.0 & Sampling temperature for response generation \\
& Max tokens (GSM8K) & 500 & Maximum tokens for GSM8K responses \\
& Max tokens (MATH/BBH) & 1024 & Maximum tokens for harder reasoning tasks \\
& Number of samples ($N$) & 1-512 & Range of Best-of-N sampling candidates \\
\midrule
\multirow{8}{*}{\shortstack[l]{\textbf{Evaluation} \\ \textbf{Configuration}}} 
& Backbone model & Llama-3.2-3B-Instruct & Base language model for response generation \\
& Reward model & OASST DeBERTa & Primary reward model for scoring responses \\
& Training dataset & GSM8K train & Dataset for training RND predictor \\
& Test datasets & GSM8K, MATH-500, BBH & Evaluation benchmarks (in-domain and OOD) \\
& Bootstrap runs & 3 & Number of bootstrap runs for confidence intervals \\
& Score normalization & Z-score & Normalization method for reward and RND scores \\
& Selection strategy & highest\_reward & Method for selecting best response from candidates \\
\midrule
\multirow{5}{*}{\shortstack[l]{\textbf{Ablation} \\ \textbf{Study}}} 
& RND weight range & 0.0-1.0 & Range of $\lambda$ values tested in weight ablation \\
& Architecture variants & 4 types & Full, simplified, embedding strategies, projection ablations \\
& Comparison baselines & BoN, RND-only & Standard Best-of-N and pessimism-only baselines \\
\bottomrule
\end{tabular}%
}
\end{table*}

%% file: tables/ablation_arch_designs.tex
\begin{table}[h]
\centering
\caption{\textbf{Summary of ablations.} Each row defines one setting and what it means in practice.}
\label{tab:ablation-structures}
\small
\setlength{\tabcolsep}{6pt}
\renewcommand{\arraystretch}{1.1}
\begin{tabular}{@{}l p{0.7\linewidth}@{}}
\toprule
\textbf{Ablation Setting} & \textbf{What it means / Rationale} \\
\midrule
\rowcolor{gray!10} \multicolumn{2}{l}{\textit{Predictor architecture}} \\
Same as Target & Predictor uses the \emph{same overall structure} as the target network (e.g., same block type and connectivity), matching width and depth. Isolates training dynamics from architectural mismatch.\\
Simplified & Predictor keeps the \emph{same hidden size and number of layers} as the target but replaces specialized target blocks with \emph{vanilla Transformer encoder blocks}. This deliberately reduces architectural complexity while preserving depth/width, aiming for better generalization and lower overfitting risk. \\
\midrule
\rowcolor{gray!10} \multicolumn{2}{l}{\textit{Embedding strategy}} \\
Shared, trainable & Predictor \emph{shares the target’s token embeddings} and \emph{updates them during training}. Pros: reuse target's pretrained semantic representations and potentially quicker convergence. Cons: tighter coupling may leak target-specific biases into the predictor. \\
Shared, frozen & Predictor \emph{reuses the target’s token embeddings} but \emph{keeps them frozen}. Pros: stable token mapping and clean isolation of predictor encoder learning. Cons: less flexibility to adapt embeddings to the predictor’s simplified blocks. \\
Separate, randomly initialized & Predictor creates \emph{its own embedding layer} with random initialization (initialized via the model’s standard weight init). Pros: full decoupling from the target, potentially better regularization. Cons: longer warm-up and higher optimization burden to reach alignment. \\
\midrule
\rowcolor{gray!10} \multicolumn{2}{l}{\textit{Projection head}} \\
No projection head & The predictor’s final hidden states are \emph{directly mapped} to the output space used for matching the target’s features. Minimal additional parameters; simplest path that reduces opportunities for overfitting. \\
Linear projection head & Adds a \emph{single linear layer} after the predictor’s hidden states and \emph{before} the output. Acts as a light adapter/bottleneck to better match target feature geometry; can improve fit at small cost in extra parameters, but may introduce overfitting. \\
\bottomrule
\end{tabular}
\end{table}